\documentclass[11pt]{article}

\usepackage[utf8]{inputenc}
\usepackage{soul}
\usepackage{makecell}
\RequirePackage{amsthm,amsmath,amsfonts,amssymb}
\RequirePackage[colorlinks,citecolor=blue,urlcolor=blue,pagebackref]{hyperref}
\RequirePackage{graphicx,subcaption,makecell,float}
\RequirePackage{xcolor, xspace}
\RequirePackage{bbm,mathtools,enumitem}
\usepackage[top=1in, bottom=1in, left=1in, right=1in]{geometry}
\hypersetup{
  colorlinks,
  linkcolor={blue!50!black},
  citecolor={blue!50!black},
  urlcolor={blue!80!black}
}

\def\Holder{H\"older\xspace}

\def\poincare{Poincar\'e\xspace}
\def\Gronwall{Gr\"onwall\xspace}
\def\Renyi{R\'enyi\xspace}

\def\reals{\mathbb{R}} %
\def\P{\mathbb{P}} 
\def\naturals{\mathbb{N}} %
\newcommand{\Gsn}{\mathcal{N}}
\def\defeq{\triangleq} %
\newcommand{\grad}{\nabla}
\newcommand{\Hess}{\nabla^2} %

\newcommand{\norm}[1]{\left\lVert#1\right\rVert}
\newcommand{\inner}[1]{{\left\langle #1 \right\rangle}} %
\newcommand{\floor}[1]{{\lfloor #1 \rfloor}} %

\def\renyi#1#2#3{R_{#1} \left( {#2} \vert {#3} \right) }
\def\finformation#1#2#3{F_{#1} \left( {#2} \vert {#3} \right)}
\def\ginformation#1#2#3{G_{#1} \left( {#2} \vert {#3} \right)}
\def\KL#1#2{\textnormal{KL}\big({#1} | {#2}\big)}
\newcommand{\wasser}[2]{\*{W}_2(#1, #2)}
\newcommand{\tv}[2]{\textnormal{TV}\left(#1, #2\right)}
\def\chisq#1#2{\chi^2\big({#1} | {#2}\big)}

\def\EE#1{\mathbb{E}\left[#1\right]}
\def\Esub#1#2{\mathbb{E}_{#2}\left[{#1}\right]}

\newcommand{\eqn}[1]{\begin{align}#1\end{align}}
\newcommand{\ind}[1]{{\mathbbm{1}}_{\{ #1 \}} }
\newcommand{\abs}[1]{\left| #1 \right|}

\def\eps{\epsilon}
\def\*#1{\ensuremath{\mathcal{#1}}}
\def\trho{\widetilde{\rho}}
\def\tx{\widetilde{x}}
\def\tf{\widetilde{f}}
\def\id{\textnormal{I}}
\newcommand{\tO}{\widetilde{\*{O}}}
\newcommand{\target}{\nu_*}
\newcommand{\lip}{L}
\newcommand{\growth}{m}
\newcommand{\pert}{b}
\newcommand{\poin}{\lambda}
\newcommand{\cond}{\kappa}
\newcommand{\step}{\eta}
\newcommand{\event}{\mathcal{E}}

\ifdefined\nonewproofenvironments\else
\ifdefined\ispres\else
\newtheorem{theorem}{Theorem}
\newtheorem{lemma}[theorem]{Lemma}
\newtheorem{corollary}[theorem]{Corollary}

\renewenvironment{proof}{\noindent\textbf{Proof.}\hspace*{.3em}}{\qed\\}
\newenvironment{proof-sketch}{\noindent\textbf{Proof Sketch}
  \hspace*{1em}}{\qed\bigskip\\}
\newenvironment{proof-idea}{\noindent\textbf{Proof Idea}
  \hspace*{1em}}{\qed\bigskip\\}
\newenvironment{proof-of-lemma}[1][{}]{\noindent\textbf{Proof of Lemma {#1}.}
  \hspace*{0em}}{\qed\\}
\newenvironment{proof-of-corollary}[1][{}]{\noindent\textbf{Proof of Corollary {#1}.}
  \hspace*{0em}}{\qed\\}
\newenvironment{proof-of-theorem}[1][{}]{\noindent\textbf{Proof of Theorem {#1}.}
  \hspace*{0em}}{\qed\\}
\newenvironment{proof-of-proposition}[1][{}]{\noindent\textbf{Proof of Proposition {#1}.}
  \hspace*{0em}}{\qed\\}
\newenvironment{proof-attempt}{\noindent\textbf{Proof Attempt}
  \hspace*{1em}}{\qed\bigskip\\}

\newenvironment{remark}{\noindent\textbf{Remark.}
  \hspace*{0em}}{\smallskip}%

\fi

\newtheorem{assumption}{Assumption}
\fi
\makeatletter
\@addtoreset{equation}{section}
\makeatother

\renewcommand*{\backref}[1]{\ifx#1\relax \else Page #1 \fi}
\renewcommand*{\backrefalt}[4]{%
  \ifcase #1 \footnotesize{(Not cited.)}%
  \or        \footnotesize{(Cited on page~#2.)}%
  \else      \footnotesize{(Cited on pages~#2.)}%
  \fi
}

\begin{document}
\mathtoolsset{showonlyrefs}

\title{Convergence of Langevin Monte Carlo in Chi-Squared and\\ \Renyi Divergence}
 \author{Murat A. Erdogdu\thanks{
  Department of Computer Science and Department of Statistical Sciences at
  the University of Toronto, and Vector Institute, \texttt{erdogdu@cs.toronto.edu}
 }
 \and
  Rasa Hosseinzadeh\thanks{
  Department of Computer Science at
  the University of Toronto, and Vector Institute, \texttt{rasa@cs.toronto.edu}
  }
  \and
  Matthew S. Zhang
  \thanks{
  Department of Computer Science at
  the University of Toronto, and Vector Institute, $\quad\quad\quad\quad\quad\quad\quad$ 
  \texttt{matthew.zhang@mail.utoronto.ca}
  }
}
\maketitle

\begin{abstract}
We study sampling from a target distribution $\nu_* = e^{-f}$ using the
unadjusted Langevin Monte Carlo (LMC) algorithm when the potential $f$
satisfies a strong dissipativity condition and it is first-order smooth with
a Lipschitz gradient.  We prove that, initialized with a Gaussian random vector that has
sufficiently small variance, iterating the LMC algorithm for $\widetilde{\mathcal{O}}(\lambda^2 d\epsilon^{-1})$
steps is sufficient to reach $\epsilon$-neighborhood of the target in
both Chi-squared and \Renyi divergence, where $\lambda$ is the logarithmic Sobolev constant of
$\nu_*$.  Our results do not require warm-start
to deal with the exponential dimension dependency in Chi-squared divergence at initialization.
In particular, for strongly
convex and first-order smooth potentials,
we show that the LMC algorithm achieves the rate estimate
$\widetilde{\mathcal{O}}(d\epsilon^{-1})$ which improves the previously
known rates in both of these metrics, under the same assumptions.  Translating this rate
to other metrics, our results also recover the state-of-the-art rate estimates in KL
divergence, total variation and $2$-Wasserstein distance in the same
setup.  Finally, as we rely on the logarithmic Sobolev inequality, our framework
covers a range of non-convex potentials that are first-order smooth and
exhibit strong convexity outside of a compact region.

\end{abstract}
\section{Introduction}\label{sec:intro}
We consider sampling from a target distribution $\target = e^{-f}$ using
the Langevin Monte Carlo (LMC)
\eqn{\label{eq:overdamped-disc}
    x_{k+1} = x_k - \step \grad f(x_k) + \sqrt{2\step} W_k,
  }
where $f:\reals^d \to\reals$ is the potential function, $W_k$ is
a $d$-dimensional isotropic Gaussian random vector independent from
$\{x_l\}_{l\leq k }$, and $\step$ is the step size.
This algorithm is the Euler discretization of the 
following stochastic differential equation (SDE)
\eqn{\label{eq:overdamped-cont}
    dz_t = -\grad f(z_t)dt + \sqrt{2}d B_t, 
}
where $B_t$ denotes the $d$-dimensional Brownian motion.  The solution of the
above SDE is referred to as the first-order Langevin diffusion, and the
convergence behavior of the LMC algorithm~\eqref{eq:overdamped-disc} is
intimately related to the properties of the diffusion
process~\eqref{eq:overdamped-cont}.  Intuitively, fast mixing of
LMC~\eqref{eq:overdamped-disc} is inherited from the Langevin
diffusion~\eqref{eq:overdamped-cont} since the Euler discretization scheme with
a sufficiently small step size ensures that the Markov chain generated by the
LMC iterations tracks its continuous counterpart.  Therefore, to ensure that the LMC
algorithm converges, one typically starts from conditions that imply
the fast convergence of the diffusion process, the Langevin dynamics given in \eqref{eq:overdamped-cont}.

Denoting the density of the Langevin diffusion $z_t$ with $\pi_t$,
the following Fokker-Planck equation
describes the evolution of the continuous
dynamics~\eqref{eq:overdamped-cont}~\cite{risken1996fokker}
\eqn{\label{eq:foker-planck}
  \textstyle 
    \frac{\partial \pi_t(x)}{\partial t}
    =
    \grad \cdot (\grad f(x) \pi_t(x)) + \Delta \pi_t(x)
    =
    \grad \cdot\left(\pi_t(x)\grad\log{\frac{\pi_t(x)}{\target(x)}}\right).
}
Convergence to the equilibrium of the above equation has been studied
extensively under various assumptions and distance measures.
Defining Chi-squared, \Renyi and Kullback–Leibler (KL) divergence measures between two
probability distributions $\rho$ and $\nu$ in $\reals^d$, respectively as
\eqn{\label{eq:chisq-kl}
  \chisq{\rho}{\nu} = &-1 + \int \left(\tfrac{\rho(x)}{\nu(x)}\right)^2 \nu(x) dx
  \ \ \text{ and } \ \
  \KL{\rho}{\nu} = \int\log\left(\tfrac{\rho(x)}{\nu(x)}\right) \rho(x) dx,\\
  &\renyi{\alpha}{\rho}{\nu}
    =
    \frac{1}{\alpha-1}\log{\int \left(\tfrac{\rho(x)}{\nu(x)}\right)^\alpha\nu(x)dx}  \ \text{ for } \ \alpha > 1,
}
the logarithmic Sobolev inequality
(LSI) is a particularly useful condition on the target $\target$,
which implies the exponential convergence of \eqref{eq:foker-planck}
in both Chi-squared and KL divergence.
A probability density $\nu$ satisfies
LSI if the following holds
\begin{align}\label{eq:LSI}\tag{LSI}
  \forall \rho,\ \     
  \KL{\rho}{\nu} \leq \frac{\poin}{2} \int \big\|{\grad \log \tfrac{\rho(x)}{\nu(x)}}\big\|^2\rho(x) dx.
\end{align}
\ref{eq:LSI} is known to hold for strongly log-concave
distributions~\cite{bakry1985LSI} in which case the constant $\poin^{-1}$ is
equal to the strong convexity constant of the potential.  This condition is
also robust against finite perturbations~\cite{holley1987LSI} which allows one
to deal with non-convex potentials (to a somewhat limited extent). If the
target $\target$ satisfies \ref{eq:LSI}, then the
distribution $\pi_t$ convergences to the target $\target$ in all three divergence measures defined in \eqref{eq:chisq-kl} exponentially fast, i.e., 
\eqn{\label{eq:contraction}
  \text{LSI} \implies
  \begin{cases}
      \KL{\pi_t}{\target} &\leq e^{-2t/\poin}\KL{\pi_0}{\target}, \\
    \chisq{\pi_t}{\target} &\leq e^{- 2t/\poin}\chisq{\pi_0}{\target}, \\
    \renyi{\alpha}{\pi_t}{\target} & \leq e^{-2t/\alpha \poin} \renyi{\alpha}{\pi_0}{\target},
  \end{cases}
}
for all $t\geq 0$.
Note that Chi-squared and $\text{KL}$ \emph{metrics} are closely related to the \Renyi divergence (e.g. $\alpha=2$ and $\alpha\to 1$ respectively);
however, convergence in different metrics (e.g. for different values of $\alpha$) 
may require different conditions on the potential (and on the target $\target$).
In fact the exponential convergence in Chi-squared divergence as in~\eqref{eq:contraction} can be
established under the \poincare inequality~\cite{cao2019exponential,chewi2020exponential}, which
holds for a wider class of potentials~\cite{bakry2008simple} (see~\cite{cao2019exponential,vempala2019rapid} for 
the convergence of $\pi_t$ in \Renyi divergence under various conditions).

Under additional smoothness assumptions on the potential function $f$, the fast
convergence of the Langevin diffusion~\eqref{eq:foker-planck} to equilibrium as
in \eqref{eq:contraction} can be translated to that of the LMC algorithm.  In
particular, implications of \ref{eq:LSI} on the convergence of LMC are
relatively well-understood in
KL divergence~\cite{dalalyan2017theoretical,
durmus2017nonasymptotic,vempala2019rapid,erdogdu2020convergence}.
In addition to \ref{eq:LSI}, assuming further that the gradient of the potential is Lipschitz continuous,
taking $\tO(d/\eps)$ steps is sufficient to reach the
$\epsilon$-neighborhood of a $d$-dimensional target distribution $\target$ in
KL divergence~\cite{vempala2019rapid}.  However, the convergence properties in stronger notions of
distance such as Chi-squared and \Renyi divergence are not explored to the same degree.  One exception
is the recent work \cite{ganesh2020faster} where authors analyzed the
convergence of LMC in $\alpha$-\Renyi divergence for strongly log-concave
targets for $\alpha >1$ and obtained the rate estimate $\tO(d/\eps^2)$,
which implies the same rate of convergence in Chi-squared divergence by setting $\alpha=2$.
More specifically, their result implies a convergence estimate of $\tO(d/\eps^2)$ in
Chi-squared divergence for strongly convex potentials that have Lipschitz gradients.

Chi-squared divergence is particularly of interest because it conveniently upper
bounds a variety of distance measures. For example, KL divergence (relative entropy), total variation (TV) distance and 2-Wasserstein ($\*{W}_2$) metrics can be upper bounded as
\eqn{\label{eq:tv-kl-chi}
  \tv{\rho}{\target} \leq \sqrt{\KL{\rho}{\target}/2} \leq \sqrt{\chisq{\rho}{\target}/2}
  \ \ \text{ and }\ \ 
  \wasser{\rho}{\target}^2/(2\poin)  \leq \chisq{\rho}{\target}.
}
For the former inequality above, see e.g.
\cite[Lemma 2.7]{tsybakov2008introduction} together with Csisz\'ar-Kullback-Pinsker
inequality~\cite{bolley2005weighted}, and the latter holds under \ref{eq:LSI}, see e.g. 
\cite[Theorem 1.1]{liu2020poincare}.
Therefore convergence in Chi-squared divergence implies
convergence in these measures of distance as well.
However, translating the rate estimate $\tO(d/\eps^2)$ obtained in \cite{ganesh2020faster} using the above
inequalities, one cannot recover the state-of-the-art convergence rates in these metrics.
For example, the rate estimate $\tO(d/\eps^2)$ in Chi-squared divergence implies
the same rate in KL divergence, which is substantially slower than the
well-known estimate $\tO(d/\eps)$ under the same assumptions, i.e., strongly convex potentials with Lipschitz gradients (see e.g. \cite{dalalyan2017theoretical,vempala2019rapid,erdogdu2020convergence}).

Our work bridges this gap in the convergence estimates, and further extends the
analysis to potentials that exhibit strong dissipativity.
Our contributions can be summarized as follows.

\begin{itemize}[noitemsep,leftmargin=.5cm]
\item  For a first-order smooth potential $f$ satisfying strong dissipativity in the
    following sense
    \eqn{
      & \norm{\grad f(x)- \grad f(y)}
      \leq
      \lip \norm{x-y},
      & \inner{x-y, \grad f(x) - \grad f(y)} 
      \geq
      \growth \norm{x-y}^2 - \pert,
    }
    where $\lip, \growth\! >\!0$ and $\pert \!\geq\! 0$, we prove that taking
    $\tO\big(\frac{\poin^2 \lip^4}{\growth^2} \times \frac{b+d}{\eps}\big)$ steps of LMC is sufficient to obtain an
    $\eps$-accurate sample from a $d$-dimensional target in both Chi-squared and \Renyi divergence,
    where $\poin$ is the \ref{eq:LSI} constant for $\target = e^{-f}$.
    Interestingly, our results do not require warm-start
    to deal with exponential dimension dependency of the Chi-squared divergence at initialization,
    and can tolerate non-convexity as long as the tails of the potential is growing quadratically fast.

\item When translated to KL divergence, TV and $\*{W}_2$ metrics using the
    inequalities \eqref{eq:tv-kl-chi}, we obtain the rate estimates 
    $\tO\big(\frac{\poin^2 \lip^4}{  \growth^2} \times \frac{b+d}{\eps}\big)$,
    $\tO\big(\frac{\poin^2 \lip^4}{ \growth^2} \times \frac{b+d}{ \eps^2}\big)$,
    and $\tO\big(\frac{\poin^3 \lip^4}{ \growth^2} \times \frac{b+d}{\eps^2}\big)$, respectively.

\item For $m$-strongly convex potentials ($\pert = 0$), we have $\poin^{-1} = \growth$; thus, our rate estimate established in
  Chi-squared divergence is able to recover the best-known rate
  estimates\footnote{Additional second order smoothness is known to speed up the
    convergence in KL divergence and $2$-Wasserstein distance
    \cite{mou2019improved, dalalyan2019user}; however, our focus in this paper is
    first-order smoothness. }
  for LMC in KL divergence, TV and $\*{W}_2$ metrics, respectively given as
    $\tO( d/\eps)$, $\tO( d/\eps^2)$,
    and $\tO( d/\eps^2)$.

\item We further discuss sampling from non-convex potentials that are covered
    by our assumptions, namely smooth potentials exhibiting strong convexity
    outside of a compact region.  By deriving bounds on their \ref{eq:LSI} constant
    $\poin$, we establish rate estimates for LMC under various scenarios.
\end{itemize}

Our analysis builds on the prominent works by
\cite{vempala2019rapid,ganesh2020faster}. More specifically, we conduct a two-phase analysis: In the first phase, we extend the analysis provided in
\cite{ganesh2020faster} to potentials that are strongly dissipative (e.g.
strongly convex outside of a compact region) and control key quantities that
impact the convergence of LMC for the interpolation process. In the second
phase, we analyze a differential inequality for the Chi-squared (and \Renyi) divergence, which
resembles the (single-phase) analysis conducted by \cite{vempala2019rapid} for
the KL divergence, to obtain our final rate estimate.  Rest of the paper is
organized as follows. We discuss related work and notation in the rest of this section.
In Section~\ref{sec:motive}, we motivate the two-phase analysis, and
state two key lemmas describing the characteristics of each phase. Section~\ref{sec:main} contains the main results on the convergence of LMC. In Section~\ref{sec:examples}, we provide examples and discuss the relative merits of certain non-convexity structures on our rate estimates. Finally in Section~\ref{sec:con},
we discuss future work. Majority of the proofs and the derivations are deferred to Appendix.

\noindent{\textbf{Related work.}}
Started by the pioneering works
\cite{durmus2016sampling,dalalyan2017theoretical,durmus2017nonasymptotic},
non-asymptotic analysis of LMC has drawn a lot of
interest~\cite{dalalyan2017furthur,
cheng2018convergence,cheng2018sharp,durmus2019high,
durmus2019analysis,vempala2019rapid,dalalyan2019user,
brosse2019tamed,li2019stochastic, erdogdu2020convergence}.
It is known that
${\tO\left(d/\eps\right)}$ steps of LMC yield an $\eps$-accurate sample in KL
divergence for strongly convex and first-order smooth potentials
\cite{cheng2018convergence,durmus2019analysis}. This is still the best
rate obtained in this setup, and recovers the fastest rates in total variation and
$2$-Wasserstein
metrics~\cite{durmus2017nonasymptotic,dalalyan2017theoretical,durmus2019high}.
For the same setting, \cite{ganesh2020faster} showed that
${\tO\left(d/\eps^2\right)}$ steps are enough for $\eps$-accurate solution in
\Renyi divergence. Recently, these global curvature assumptions are relaxed to
growth conditions~\cite{cheng2018sharp,erdogdu2018global}. For example,
\cite{vempala2019rapid} established convergence guarantees for LMC when
sampling from targets distributions that satisfy a log-Sobolev inequality, and
has a smooth potential. This corresponds to potentials with quadratic
tails~\cite{bakry1985LSI,bobkov1999exponential} up to finite
perturbations~\cite{holley1987LSI}; thus, this result is able to deal with
non-convex potentials while achieving the same rate of convergence
${\tO\left(d/\eps\right)}$ in KL divergence.
Finally, convergence of zigzag samplers is established in Chi-squared divergence
under a warm-start condition, in order to deal with the ill behavior of this metric
at initialization~\cite{lu2020complexity}.

\noindent{\textbf{Notation.}}
Throughout the paper, $\log$ denotes the natural logarithm.  For a real number
${x \in \reals}$, we denote its absolute value with $\abs{x}$.  We denote the
Euclidean norm of a vector ${x \in \reals^d}$ with $\norm{x}$.  The gradient,
divergence, and Laplacian of $f$ are denoted by
${\grad f(x)}$, $\grad \cdot f(x)$ and $\Delta f(x)$, respectively.

We use $\EE{ x}$ to denote the expected value of a random variable or a vector
$x$, where expectations are over all the randomness inside the brackets.
For probability densities $p$,$q$ on $\reals^d$, we use $\KL{p}{q}$,
$\chisq{p}{q}$, and $\renyi{\alpha}{p}{q}$ to denote their KL (or relative entropy),
Chi-squared, and \Renyi divergence (for $\alpha>1$), respectively, which are defined in \eqref{eq:chisq-kl}.
To ease the notation, we often use $\frac{p}{q}(x)$ instead of ${p(x)}/{q(x)}$.
We denote the Borel $\sigma$-field of $\reals^d$ with
${\mathcal{B}(\reals^d)}$.  $2$-Wasserstein metric and the total variation (TV) distance
are defined respectively as
\begin{align*}\textstyle
   \wasser{p}{q} = \inf_{\nu} \left(\int \norm{x-y}^2 d \nu(p,q)\right)^{{1}/{2}}\!\!\!\!,
\  \text{ and }\
   \tv{p}{q} = \sup_{A \in \mathcal{B}(\reals^d)} \abs {\int_A (p(x)-q(x))dx},
\end{align*}
where in the first formula, infimum runs over the set of probability measures
on ${\reals^d \times \reals^d}$ that has marginals with corresponding densities
$p$ and $q$.  Multivariate Gaussian distribution with mean $\mu \in \reals^d$
and covariance matrix $\Sigma \in \reals^{d\times d}$ is denoted with
$\Gsn(\mu, \Sigma)$.

%
\section{Two-Phase Analysis: Motivation and Assumption}\label{sec:motive}
A representative analysis of LMC (see e.g.
\cite{vempala2019rapid,erdogdu2020convergence}) starts with the interpolation
process,
\eqn{\label{eq:overdamped-inter}
    d\tx_{t} = -\grad f(x_{\floor{t/\step}})dt + \sqrt{2}dB_t
    \ \ \text{ with }\ \ 
    \tx_{0} = x_0.
}
Notice that the drift $\grad f(x_{\floor{t/\step}})$ of the above
process is evaluated at the LMC iterate $x_\floor{t/\step}$, and it is constant
within each interval $t\in [k\step, (k+1)\step)$ for an integer $k$.
Therefore, with the right coupling of the Brownian motion in
\eqref{eq:overdamped-inter} and the additive Gaussian in the LMC
update~\eqref{eq:overdamped-disc}, the solution of \eqref{eq:overdamped-inter}
produces the LMC iterates for $t=\step k$ (i.e. $\tx_{k\step} = x_k$) by
simply interpolating the discrete algorithm to a continuous-time process.  We denote
by $\trho_{t}$ and $\rho_k$, the distributions of $\tx_{t}$ and $x_k$,
and we observe easily that $\trho_{k\step} = \rho_{k}$.  The advantage of
analyzing the interpolation process is in its continuity in time, which allows
one to work with the Fokker-Planck equation.  Using this property in
Lemma~\ref{lem:disc-fok-plan-1}, we show that the time derivative of
$\chisq{\trho_{t}}{\target}$ differs from the corresponding differential
inequality for the continuous-time process by an additive error term.

In order to obtain a differential inequality in Chi-squared divergence, 
it is sufficient if the target satisfies a \poincare inequality (PI),
which is given as
\eqn{\label{eq:PI} \tag{PI}
  \forall \rho,\ \     
  \chisq{\rho}{\target} \leq \poin \int \big\|{\grad \tfrac{\rho(x)}{\target(x)}}\big\|^2\target(x) dx.
}
We note that the above condition holds under \ref{eq:LSI} with the same constant $\poin$
\cite{villani2003topics}, and emphasize that our final convergence results (in both Chi-squared and \Renyi divergence) require the stronger condition \ref{eq:LSI} even though \ref{eq:PI} suffices for the following lemma.

\begin{lemma}\label{lem:disc-fok-plan-1}
  If $\target$ satisfies \ref{eq:PI}, then the following inequality governs
  the evolution of the Chi-squared divergence of the interpolated process
  \eqref{eq:overdamped-inter} from the target
  \eqn{\label{eq:disc-fok-plan-1}
    \frac{d}{d t} \chisq{\trho_{t}}{\target} 
    \leq
    -\frac{3}{2\poin} \chisq{\trho_{t}}{\target} 
    +2 \EE{\tfrac{\trho_{t}}{\target}(\tx_{t})^2}^{1/2}\!\!\! \EE{\norm{\grad f(\tx_{t}) - \grad f(x_{\floor{t/\step}})}^4}^{1/2}\!\!\!.
  }
\end{lemma}
\begin{proof}
  The proof follows from similar lines that lead to a differential inequality
  in KL divergence (see for example~\cite{vempala2019rapid}). Let
  $\trho_{t|k}$ denote the distribution of $\tx_{t}$ conditioned on $x_k$
  for $k=\floor{t/\step}$, which satisfies
  \eqn{
    \frac{\partial \trho_{t|k} (x) }{\partial t}
    =
    \grad \cdot(\grad f(x_k) \trho_{t \vert k}(x)) + \Delta \trho_{t\vert k}(x).
  }
  Taking expectation with respect to $x_k$ we get
  \eqn{\label{eq:time-derivative-trho}
    \frac{\partial \trho_{t} (x) }{\partial t}
    &=
    \grad \cdot \left(
      \trho_{t}(x) \left(
        \EE{\grad f(x_k) - \grad f(x)\vert \tx_{t}=x} + 
        \grad \log\left(\frac{\trho_{t}(x)}{\target(x)}\right)
      \right)
    \right).
  }
  Now we consider the time derivative of Chi-squared divergence of
  $\trho_{t}$ from the target $\target$
\begin{align}
    \frac{d}{d t} \chi^2\big( \trho_{t} \vert \target \big)
    &=
    2 \int \frac{\trho_{t}(x)}{\target(x)} \times
    \grad \cdot \left(
      \trho_{t}(x) \left(
        \EE{\grad f(x_k) - \grad f(x)\vert \tx_{t}=x} + 
        \grad \log\left(\frac{\trho_{t}(x)}{\target(x)}\right)
      \right)
    \right)dx
  \\
    &\stackrel{1}{=}
    -2 \int \trho_{t}(x) \Big\langle\grad\frac{\trho_{t}(x)}{\target(x)},
      \EE{\grad f(x_k) - \grad f(x)\vert \tx_{t}=x} + 
      \grad \log\left(\frac{\trho_{t}(x)}{\target(x)}\right)
      \Big\rangle dx
    \\
    &\stackrel{2}{\leq}
    -\frac{3}{2}\int\bigg \|\grad\frac{\trho_{t}(x)}{\target(x)}\bigg\|^2 \target (x)dx
    +2 \int \EE{\frac{\trho_{t}(x)}{\target(x)}\norm{\grad f(x) - \grad f(x_k)}^2\vert \tx_{t}=x}\trho_{t}(x)dx
    \\
    &\stackrel{3}{\leq}   
    -\frac{3}{2\poin} \chisq{\trho_{t}}{\target} 
    +2 \EE{\frac{\trho_{t}(\tx_{t})}{\target(\tx_{t})}\norm{\grad f(\tx_{t}) - \grad f(x_k)}^2},
  \end{align}
  where step~1 follows from the divergence theorem, step~2 from
  $\inner{a,b} \leq \frac{1}{4}\norm{a}^2 + \norm{b}^2$ and in step~3,
  we used \ref{eq:PI} with $\rho$ replaced by $\trho_t$.  Finally, the result follows from the Cauchy-Schwartz
  inequality on the second term.
\end{proof}
The above differential inequality
will be used to establish a single step bound that can be iterated to yield the
final convergence result. For this, one needs to control $i)$ the additive
error term in~\eqref{eq:disc-fok-plan-1}, namely
$\mathbb{E}[\|\grad f(\tx_{t}) - \grad f(x_\floor{t/\step})\|^4]$
under a smoothness condition on
the potential function, and $ii)$ the expected squared ratio of densities
$\mathbb{E}[\tfrac{\trho_{t}}{\target}(\tx_{t})^2]$ which is harder to
bound -- indeed, it is exponential in dimension at initialization. Therefore,
in order to avoid any warm-start assumption, we conduct our convergence
analysis in two phases.  In the first phase, we show that after taking $N$
steps of LMC, the expected squared ratio (over the interpolation process) is bounded by an absolute constant at time $N\eta$, and
moreover it stays uniformly bounded for the time interval $[N\eta, 2N\eta]$. That is,
\eqn{\label{eq:warm}
    \EE{\tfrac{\trho_{T}}{\target}(\tx_{T})^2}\leq B
    \ \ \ \ \text{ for }  \ \ \ \   
    N\step \leq T \leq 2N\step, 
}
where $B$ is an absolute constant. The above opaque condition would ultimately imply that the LMC iterates also stay
warm when the iteration counter belongs to the interval $[N,2N]$.

Before describing the second phase of the analysis, we illustrate the above
phenomenon on a simple Gaussian example where the expected squared ratio is
exponential in dimension at initialization; yet, it stabilizes exponentially
fast with the number of LMC iterations.

\noindent{\textbf{Motivating Example.}} In this toy example, the above uniform warmness
condition~\eqref{eq:warm} is verified for sampling from a Gaussian target $e^{-f}$
using LMC with step size $\eta$. Assume for simplicity that $f(x) = \frac{1}{2}\|x\|^2 + C$,
where $C$ is the normalizing constant
and $x_0 \sim \Gsn(0,\sigma_0^2 \id)$.
For an integer $k\geq0$,
$t\in[0,\step)$, and $T = k\step +t$, it is easy to compute the distribution of the interpolation process
\eqn{
    \trho_{T} = \Gsn(0, \sigma_{T}^2 \id)
    \ \text{ where }\ 
    \sigma_{T}^2 \coloneqq (1-t)^2 \sigma^2_{k\step} + 2t
    \ \text{ and }\ 
    \sigma_{k\step}^2 \coloneqq (1-\step)^{2k} \sigma_0^2
    + \tfrac{1- (1-\step)^{2k}}{1-\step/2}.
}
Moreover, elementary calculations yield that the expected squared ratio is given as
\eqn{\label{eq:warm-2}
    \EE{\tfrac{\trho_{T}}{\target}(\tx_{T})^2}
    =
    \tfrac{1}{(3/\sigma^2_{T} - 2)^{d/2} \sigma^{3d}_{T}}
    \ \ \text{ whenever }\ \ 
    \sigma^2_{T} < \tfrac{3}{2}.
}
For simplicity, let us initialize with $\sigma_0^2 = 0.5 / ( 1- \step/2)$.  For
a sufficiently small step size $\step$, one can verify that $\sigma_{k\step}^2$
is monotonically increasing and converges to $1/(1-\step/2) < 1.5$.
Moreover, $\sigma_0^2 > 1/2$ bounded away from $0$.  Therefore, at
initialization (for $k=0$) we have
$\mathbb{E}\big[{\tfrac{\trho_{T}}{\target}(\tx_{T})^2}\big]=e^{\*O(d)}$.
However, notice that if at any point along the iterations, the condition $1 -
\frac{1}{d} \leq \sigma^2_{k\step} \leq 1+\frac{1}{d}$ is satisfied, then we
can show
$\mathbb{E}\big[{\tfrac{\trho_{T}}{\target}(\tx_{T})^2}\big]=\*O(1)$ in
the subsequent iterations, and accordingly $B$ in \eqref{eq:warm} becomes
$\*O(1)$.  Note that the denominator $\sigma \to (3/\sigma^2-2)^{d/2} \sigma^{3d}$ attains
its minimum value in the interval $[1-\frac{1}{d}, 1+\frac{1}{d}]$ on its boundary (assuming
$d>3$). If this condition holds, the resulting upper bound on its inverse becomes
$\*O(1)$.  On the other hand, reaching $1-\frac{1}{d} \leq \sigma_{T}^2 \leq 1+
\frac{1}{d}$ in this setting is exponentially fast, which suggests conducting a
two-phase analysis. First, we prove that the distribution $\trho_T$
gets close to target $\target$ so that their expected squared ratio reduces to
$\*O(1)$, and stays warm in the subsequent iterates. Then in the next phase, we
proceed the analysis with the differential inequality in Lemma~\ref{lem:disc-fok-plan-1} to obtain the final
convergence estimate.

To formalize the above argument, we make the following assumptions on the potential function.
\begin{assumption}
\label{as:all}
The potential function $f$ is first-order smooth and strongly dissipative, i.e., $\forall x,y$
\eqn{
    & \norm{\grad f(x)- \grad f(y)}
    \leq
    \lip \norm{x-y},
    & \inner{x-y, \grad f(x) - \grad f(y)} 
    \geq
    \growth \norm{x-y}^2 - \pert,
}
for some constants $L,m>0$ and $b\geq0$, where we also define the condition number as $\cond\coloneqq \lip/\growth$.
\end{assumption}
The strong dissipativity condition is equivalent to \cite[Assumption
3]{cheng2018sharp}, but it is presented in this form for convenience with later
calculations.  The target satisfies a \ref{eq:LSI} under strong dissipativity
which can be easily deduced from \cite{bakry1985LSI,cattiaux2010note} (see Section~\ref{sec:lsi-all},
cf. \cite[Prop~3.2]{raginsky2017non}). For strongly convex
potentials ($\pert=0$), by the Bakry-{\'E}mery criterion we have $\poin = 1/\growth$,
i.e., the \ref{eq:LSI} constant is equal to the inverse of the strong convexity
parameter of the potential.
By plugging in $y=0$, elementary algebra reveals that the strong dissipativity
implies the standard $2$-disspativity
condition $\inner{x,\grad f(x)} \geq \growth'\|x\|^2-\pert'$, which is commonly employed
in recent analyses in sampling and non-convex
optimization~\cite{raginsky2017non,yu2020analysis,erdogdu2020convergence}.
While the stronger version does not cover all the potentials covered
by the standard dissipativity, it still allows for finite perturbations (similar to
2-dissipativity and \ref{eq:LSI}), which we discuss in detail in
Section~\ref{sec:examples}.

\vspace{.1in}
\noindent{\textbf{First phase.}} We establish the following bound on the expected squared
ratio.
\begin{lemma}\label{lem:uniform-warmness}
  For $\alpha> 1$
  and for a potential $f$ satisfying
  Assumption~\ref{as:all},  initialize the LMC algorithm with $x_0 = \mathcal{N}(0, \sigma^2 I)$ for $\sigma^2 < (1+L)^{-1}$. 
  If the step size satisfies $\step \leq
  \tfrac{2}{\norm{\grad f(0)}^2} \wedge
  \tfrac{1 \wedge \growth}{4(1 \vee \lip^2)}$
  and for some absolute constant $c$, the following conditions hold 
  \eqn{
    c\alpha^2 \cond^2 \lip^2 N (\pert + d+ \log{N}) \step^2 \leq 1
    \ \ \ \text{ and }\ \ \ 
    N\step \geq 2\alpha \poin\log(4\alpha dC_{\sigma}),
  }
  for the dimension free constant $C_\sigma = 1+\tfrac{f(0) +\|\grad f(0)\|^2}{d} - \log(\sigma^2[(1+L)\wedge 2\pi])$,
  then we have
  \eqn{
    \EE{\frac{\trho_T}{\target}(\tx_T)^{2\alpha-2}} \leq 14 \alpha^{\frac{1}{4}}, \qquad \forall T \in [N\step, 2N\step].
  }
\end{lemma}
The above result extends the results of \cite{ganesh2020faster} to the interpolation process, 
and it is key to our analysis. Its proof is deferred to Section~\ref{sec:proof-main}. We use the above bound for $\alpha=2$ for the Chi-squared divergence, but the case $\alpha>2$ will be useful when we extend the results to the \Renyi divergence. 
After taking $N$ iterations of LMC, the expected density ratio
$\EE{\frac{\trho_T}{\target}(\tx_T)^{2\alpha-2}} $ is bounded and it stays bounded in the subsequent $N$ iterations,
as in \eqref{eq:warm}.

\vspace{.1in}
\noindent{\textbf{Second phase.}} In the next stage of the analysis, we adapt the strategy
used in~\cite{vempala2019rapid} to the differential
inequality in Chi-squared divergence only for the iterations ranging from step $N$ to $2N$.
Since for these iterations, we have a uniform bound on the expected squared
ratio over the interpolation process (by Lemma~\ref{lem:uniform-warmness}),
we can simply integrate the differential inequality in
Lemma~\ref{lem:disc-fok-plan-1} to obtain the following single step bound.

\begin{lemma}
\label{lem:sing-step} 
Instantiate the assumptions of Lemma~\ref{lem:uniform-warmness} for $\alpha=2$.  If further
$\step \leq \tfrac{\poin}{2}$, then for any iteration number $k$ such that
$2N\geq k \geq N$,
the following inequality controls the evolution of LMC in Chi-squared divergence
\eqn{\label{eq:sing-step}
    \chisq{\rho_{k+1}}{\target} 
    \leq
    \bigg(1-\frac{3\step}{4\poin}\bigg)\chisq{\rho_{k}}{\target} + c \beta \lip^2 (\pert + d) \step^2,
}
where $c$ is an absolute constant and $\beta$ is a dimension free parameter defined as
\eqn{
    \beta^2 \coloneqq 
    1 + 
    \frac{\sigma^4(1+2/d)+6\sigma^2\norm{x^*}^2/d+\norm{x^*}^4/d^2}{(1+\pert/d)^2} +
    \frac{\sigma^2 + \norm{x^*}^2/d}{1+\pert/d}.
}
\end{lemma}
The proof follows by integrating the differential inequality derived in Lemma~\ref{lem:disc-fok-plan-1} for $t\in[0,\eta]$ after using the uniform bound on the expected squared ratio given by Lemma~\ref{lem:uniform-warmness}, which we defer to Section~\ref{sec:proof-main}.  The key innovation of the above result is that the additive error term in \eqref{eq:sing-step}, the second term on the
right hand side, has $\*{O}(d\eta^2)$ dependence for the LMC iterations ranging from $N$
to $2N$. This is because the constant $\beta$ is uniformly bounded in dimension.


%
\section{Main Results}\label{sec:main}
In this section, we first provide results
on the convergence of LMC \eqref{eq:overdamped-disc} in
Chi-squared divergence, by simply iterating the single step bound in
Lemma~\ref{lem:sing-step} from iteration $N$ to $2N$.
Then, using a different differential inequality but similar arguments, we extend the convergence to the \Renyi divergence.
Before we present the main technical results of this
paper, we note that the following (unnormalized) potential
$$
	f(x) = \tfrac{1}{2}\|x\|^2 + \tfrac{5}{4}\cos(\|x\|)
$$
serves as a canonical example for our framework. This non-convex potential
satisfies strong dissipativity and it has a Lipschitz gradient; thus, Assumption
\ref{as:all} is satisfied. The resulting target also satisfies \ref{eq:LSI} with constant $e^5$.
We will discuss several non-trivial examples in Section~\ref{sec:examples}.

\begin{theorem}\label{thm:main}
Let the potential $f$ satisfy Assumption~\ref{as:all} and suppose we run $2N$ iterations of LMC \eqref{eq:overdamped-disc} with step size $\eta$ to sample from $\target = e^{-f}$.
If we initialize $x_0$
with $\Gsn(0, \sigma^2 \id)$ for some $\sigma^2 <(1+\lip)^{-1},$ in order to get
$\chisq{\rho_{2N}}{\target} \leq \eps,$ it is sufficient if the following inequalities hold
\begin{align}\label{eq:main-ineqs}
\step
&\leq
\frac{2}{\norm{\grad f(0)}^2}
\wedge
\frac{1 \wedge \growth}{4(1 \vee \lip^2)}
\wedge
\frac{\poin}{2} 
\wedge
\frac{c_3}{\beta\poin \lip^2} \times \frac{\eps}{\pert + d}
\\
N \step
&\geq
c_2 \poin \log \big({144 dC_\sigma}/{\eps}\big)
\\
c_1& \geq 
\cond^2 \lip^2(\pert + d + \log{N}) N \step^2 ,\label{eq:main-ineqs-3}
\end{align}
where $\lambda$ is the \ref{eq:LSI} constant of the target, $c_1,c_2,c_3$ are absolute constants, and $C_\sigma$ and $\beta$
are dimension free constants defined respectively in
Lemmas~\ref{lem:uniform-warmness} and \ref{lem:sing-step}.

Consequently, if we choose
\begin{align}
  \label{eq:num-steps-chisq}
  &\step = \frac{c_3 }{\beta} \times \frac{m^2}{ \poin \lip^4 \log(\frac{L^4\poin^2}{m^2})} \times \frac{\eps}{\pert + d}\times \frac{1}{\log(144 dC_\sigma /\eps)^2}\\
  &N = \frac{c_2\beta}{c_3} \times \frac{ \poin^2 \lip^4 \log(\frac{\poin^2L^4}{m^2})}{m^2} \times \frac{ \pert + d}{\eps}\times\log({144 dC_\sigma }/{\eps})^3,
\end{align}
then $2N$ iterations of LMC yield $\chisq{\rho_{2N}}{\target} \leq \eps$ for $N\geq 2$ and
for $\eps>0$ sufficiently small.
\end{theorem}

The above theorem implies that $\tO(\lambda^2 L^4/m^2 \times d/\eps)$ steps of LMC algorithm
is sufficient to obtain an $\eps$-accurate sample in Chi-squared divergence.
The rate estimate given in \eqref{eq:num-steps-chisq} can be slightly improved to $\mathcal{O}(d/\eps \log(d/\eps) \log(d)^2)$ in terms of $\eps$ dependency at the expense of introducing more complicated expressions; yet, we use \eqref{eq:num-steps-chisq} for simplified exposition.

The theorem states that if the step size and the number of
iterations satisfy the three conditions given in \eqref{eq:main-ineqs}, the LMC
algorithm is guaranteed to produce an $\eps$-accurate sample in exactly
$2N$ iterations. We emphasize that this result may not hold for the subsequent
iterates of LMC because of the last condition~\eqref{eq:main-ineqs}. This is
due to the delicate bound we construct using Lemma~\ref{lem:uniform-warmness},
which will be violated as $N\to \infty$ for any fixed step size $\step$.
Rate estimates with this restriction are frequent in the literature~\cite{shen2019randomized,ganesh2020faster,erdogdu2020convergence}.
This typically occurs when there is a diverging bound on a quantity that
appears in the convergence analysis; in \cite{erdogdu2020convergence} this is due to
the diverging moment estimates, in \cite{ganesh2020faster} it is due to the \Renyi divergence
between the clipped LMC and the clipped Langevin diffusion, and
in \cite{shen2019randomized} it is due to the sum of the expected squared difference between 1-step exact Langevin diffusion and the corresponding LMC iterates (ergodicity of the algorithm is  established in another work~\cite{he2020ergodicity}).
In our case,
the source of this is the expected squared ratio bounded in Lemma~\ref{lem:uniform-warmness}.
We note that LMC
is ergodic~\cite{mattingly2002ergodicity} and its subsequent iterates after iteration $2N$
remain $\eps$-accurate in KL divergence and $2$-Wasserstein distance
from the target under \ref{eq:LSI}~\cite{vempala2019rapid}.
Therefore it is likely that
the case $N \to \infty$ can be covered with a different proof technique;
but the analysis used in the current paper introduces this artifact.

%
%
%

We can translate our rate estimate in Chi-squared divergence using \eqref{eq:tv-kl-chi},
to obtain guarantees in  KL divergence, TV, and $\*{W}_2$ metrics.

\begin{corollary}
  \label{cor:translate}
  Instantiate the assumptions and the notation in Theorem~\ref{thm:main}.
Table~\ref{tab:lsi-conv-rate} summarizes the convergence rate estimates in various measures of
distance.
\begin{table}[H]
\renewcommand{\arraystretch}{1.5}
  \begin{center}
    \begin{tabular}{c|c|c|c}
        \textsc{Distance} & $\eps_{\chi^2}$ & $N$  & $\step$\\ 
      \hline
        $\chi^2$ &$\eps$ &$\tO\Big(\frac{\poin^2\lip^4}{\growth^2} \times \frac{b+d}{\eps}\Big)$ & $\tO\Big(\frac{\growth^2}{\poin\lip^4} \times \frac{\eps}{b+d}\Big)$\\
      \hline
      KL &$\eps$ &$\tO\Big(\frac{\poin^2\lip^4}{\growth^2} \times \frac{b+d}{\eps}\Big)$ & $\tO(\frac{\growth^2}{\poin\lip^4} \times \frac{\eps}{b+d}\Big)$\\
      \hline
      TV &$2 \eps^2$ &$\tO\Big(\frac{\poin^2\lip^4}{\growth^2} \times \frac{b+d}{\eps^2}\Big)$ & $\tO\Big(\frac{\growth^2}{\poin\lip^4} \times \frac{\eps^2}{b+d}\Big)$\\
      \hline
      $\*{W}_2$ &$\eps^2/2\poin$  &$\tO\Big(\frac{\poin^3\lip^4}{\growth^2} \times \frac{b+d}{\eps^2}\Big)$ & $\tO\Big(\frac{\growth^2}{\poin^2\lip^4} \times \frac{\eps^2}{b+d}\Big)$
    \end{tabular}
    \vspace{.1in}
    \caption{Translation of the rate estimate in Chi-squared divergence to various measures of distance by choosing an appropriate accuracy level $\eps_{\chi^2}$ in Theorem~\ref{thm:main}.
    \label{tab:lsi-conv-rate}}
  \end{center}
\end{table}
\end{corollary}

For strongly convex potentials (i.e. $\lambda = 1/m$ and $b=0$),
the above rate estimates recover the state-of-the-art
estimates in all of the above measures of distance in both accuracy $\eps$ and
dimension $d$. In other words, there is no loss in converting the rates using
the inequalities \eqref{eq:tv-kl-chi}. However, we note that the known
condition number dependency is $\tO{(\kappa^2)}$, in for example KL divergence~\cite{vempala2019rapid}, and our
estimate produces $\tO{(\kappa^4)}$; thus, there is room for improvement in
this dependency.

\subsection{Extending convergence to the \Renyi Divergence for $\alpha>1$}\label{sec:renyi}
The results presented in the previous section for the Chi-squared divergence can be extended to the \Renyi divergence with minimal effort.
The key is to establish a differential inequality in this measure of distance as given below (cf. Lemma~\ref{lem:disc-fok-plan-1}), which will be solved and iterated to yield a convergence rate in the \Renyi divergence. Contrary to Lemma~\ref{lem:disc-fok-plan-1} which was established under \ref{eq:PI}, the following result is established under \ref{eq:LSI}.
\begin{lemma}\label{lem:renyi-diff-ineq}
  If $\target$ satisfies \ref{eq:LSI} and $\alpha> 1$, then the following inequality
  controls the evolution of the $\alpha$-\Renyi divergence of the interpolated process from the target
  \eqn{\label{eq:renyi-diff-ineq-1}
    \frac{d}{dt} \renyi{\alpha}{\trho_t}{\target}
    \leq
    -\frac{3}{2 \alpha \poin} \renyi{\alpha}{\trho_t}{\target}
    +
    \alpha \EE{\tfrac{\trho_t}{\target}(\tx_t)^{2\alpha-2}}^{\frac{1}{2}} \EE{\norm{\grad f(\tx_t) - \grad f(x_{\floor{t/\step}})}^4}^{\frac{1}{2}}.
  }
\end{lemma}
The proof of the above statement is similar to that of Lemma~\ref{lem:disc-fok-plan-1}, and deferred to Section~\ref{sec:renyi-proofs}.
To iterate the bound obtained using Lemma~\ref{lem:renyi-diff-ineq},
we again conduct a two-phase analysis. In the first phase, we use Lemma~\ref{lem:uniform-warmness};
$N$ steps of LMC implies that $\EE{\frac{\trho_t}{\target}(\tx_t)^{2\alpha-2}}$ is bounded by $\mathcal{O}(\alpha^{0.25})$, and stays bounded in the subsequent $N$ iterations.
In the second phase, we use
the following generalization of the single-step bound in Lemma~\ref{lem:sing-step} for the \Renyi divergence.
\begin{lemma}\label{lem:renyi-final-bound}
  Under the assumptions of Lemma \ref{lem:uniform-warmness}, and if we additionally have $\step \leq \frac{2\alpha \poin}{3}$, then for any iteration $k$ such that $k \in [N, 2N]$, LMC satisfies the following bound
  \eqn{\label{eq:sing-step-renyi}
    \renyi{\alpha}{\rho_{k+1}}{\target} \leq \left(1-\frac{3\step}{4\alpha\poin }\right) \renyi{\alpha}{\rho_k}{\target} + c\beta \lip^2 (\pert + d)\alpha^{9/8} \step^2,
  }
  where $c$ is an absolute constant and $\beta$ is defined in Lemma \ref{lem:sing-step}.
\end{lemma}

The immediate consequence of this lemma is a bound on the \Renyi divergence, which is stated in the following theorem.

\begin{theorem}\label{thm:main-renyi}
  For $\alpha> 1$ and for a potential $f$ satisfying Assumption~\ref{as:all}, suppose we run $2N$ iterations of LMC \eqref{eq:overdamped-disc} with step size $\eta$ to sample from $\target = e^{-f}$.
  If we initialize $x_0$
  with $\Gsn(0, \sigma^2 \id)$ for some $\sigma^2 <(1+\lip)^{-1},$ in order to get
  $\renyi{\alpha}{\rho_{2N}}{\target} \leq \eps,$ it is sufficient if the following inequalities hold
  \begin{align}\label{eq:main-ineqs-renyi}
    \step
    &\leq
      \frac{2}{\norm{\grad f(0)}^2}
      \wedge
      \frac{1 \wedge \growth}{4 (1 \vee \lip^2)}
      \wedge
      \frac{2\alpha \poin}{3}
      \wedge
      \frac{c_3}{\beta\lambda \lip^2 \alpha^{17/8}} \times \frac{\eps}{(\pert + d)}
    \\
    N \step
    &\geq
      c_2 \alpha \poin \log \Big(\tfrac{ 6\alpha^2dC_\sigma }{(\alpha-1)\eps}\Big)
    \\
    c_1& \geq \alpha^2
         \cond^2 \lip^2(\pert + d + \log{N}) N \step^2 ,\label{eq:main-ineqs-renyi-3}
  \end{align}
  where $\poin$ is the \ref{eq:LSI} constant of the target, $c_1,c_2,c_3$ are absolute constants, and $C_{\sigma}$
  and $\beta$
  are dimension free constants defined respectively in
  Lemmas~\ref{lem:uniform-warmness} and \ref{lem:sing-step}.

  Consequently, if we choose
  \begin{align*}
    &\step = \frac{c_3 }{\beta}
      \times \frac{m^2}{L^4 \poin \log\big(\frac{L^4\poin^2}{m^2}\big)}
      \times \frac{1}{\alpha^3}
      \times \frac{\eps}{b+d}
      \times\frac{1}{\log  \big( \tfrac{6\alpha^2 dC_\sigma}{(\alpha-1)\eps}\big)^2} \\
    &N =  \frac{c_2\beta}{c_3}
      \times \frac{ \lip^4 \poin^2 \log\big(\frac{L^4\lambda^2}{m^2}\big)}{m^2}
      \times \alpha^{4}
      \times \frac{ \pert + d}{\eps} 
      \times \log \big( \tfrac{6\alpha^2 dC_\sigma}{(\alpha-1)\eps}\big)^3,
  \end{align*}
  then, $2N$ steps of LMC yield $\renyi{\alpha}{\rho_{2N}}{\target} \leq \eps$ for $N \geq 2$ and for a sufficiently small $\eps >0$.
\end{theorem}

The above theorem is similar to Theorem~\ref{thm:main};
therefore the same remarks also apply to this result.
One important difference is the $\alpha$ dependency of the rate
$\mathcal{O}{(\alpha^4\log(\frac{\alpha^2}{\alpha-1})^{3})}$,
which
diverges as $\alpha \to 1$ and $\alpha \to \infty$. If one is interested in $\alpha \approx1$,
then using the monotonicity of \Renyi divergence, 
one can obtain better rate estimates, for example by bounding $R_\alpha$ by $R_2$.

The above result also implies the same rate estimate for the Chi-squared divergence;
however, several key steps for the latter require milder conditions (cf. Lemmas~\ref{lem:disc-fok-plan-1} and \ref{lem:renyi-diff-ineq}).
Therefore relaxing the conditions of Lemma~\ref{lem:uniform-warmness} would directly improve
the set of feasible potentials in the Chi-squared divergence, which is not true for the \Renyi divergence due to the conditions of Lemma~\ref{lem:renyi-diff-ineq}.
Also see Section~\ref{sec:con} for a detailed discussion on this direction.


%

\section{Examples}\label{sec:examples}
Assumption~\ref{as:all} implies that the target satisfies
a log-Sobolev inequality.
This can be deduced from the results of
\cite{bakry1985LSI,cattiaux2010note},
and a derivation is provided in Section~\ref{sec:lsi-all} (cf. \cite[Prop~3.2]{raginsky2017non}).
However, under more specific curvature conditions on the potential,
one can obtain better estimates (in terms of dimension) on the \ref{eq:LSI} constant.
We discuss a few interesting cases below.

\subsection{Strongly convex and first-order smooth potentials}
Potentials that are in this category satisfy $\lip \id \succeq \grad^2 f(x)
\succeq \growth \id$.  It is easy to see that Assumption~\ref{as:all} holds for
the same parameters $L, m$, and $\pert = 0$. Moreover, due to the
Bakry-{\'E}mery criterion, the target $\target = e^{-f}$ satisfies \ref{eq:LSI}
with constant $\poin = 1/\growth$.

While strongly convex potentials have been most frequently studied in prior
work, the known rate in Chi-squared and \Renyi divergence is $\*{O}(d/\eps^2)$~
\cite{ganesh2020faster}. Our analysis instead obtains $\*{O}( d/\eps)$;
this represents a significant improvement to the known convergence rate in
these metrics.  When translated to KL divergence (using \eqref{eq:chisq-kl}),
our rate recovers the state-of-the-art rates for the same class of potentials~\cite{dalalyan2017theoretical,
  durmus2017nonasymptotic}.

Despite their apparent simplicity, this function class contains numerous
practical applications.

\textbf{Ex 1: Gaussian mixtures.} In this case, we consider sampling from
potentials of the form
\eqn{
    \target(x) \propto \exp\big(-\tfrac{1}{2}\norm{x - a}^2\big) + \exp\big(-\tfrac{1}{2} \norm{x + a}^2\big)
}
with $a \in \reals^d$ a parameter controlling the modal separation. These
potentials have strongly convex densities if $\norm{a}_2 < 1$, with first,
second and third order derivatives all being Lipschitz
\cite{dalalyan2017theoretical}.

\textbf{Ex 2: Bayesian logistic regression}. Consider data samples
$V = \{v_i\}_{i=1}^n \in \reals^{n \times d}, y = \{y_i\}_{i=1}^n \in \reals^n$, and
a Bernoulli distribution $\P(y=1|v) = 1/(1+\exp(-\inner{x, v}))$ with parameter
$x \in \reals^d$. If we use the prior $x \sim \Gsn(0, \alpha\Sigma_V^{-1})$
where $\Sigma_V = V^\top V/n$, then the resulting posterior is
\eqn{
    \target(x) \propto \exp \Big( y^\top W x - \sum_{i=1}^n \log(1+\exp(\inner{x, v_i})) - \frac{\alpha}{2} \norm{\Sigma_V^{1/2} x}^2  \Big).
}
Again, the potential is strongly convex, with Lipschitz derivatives up to the
third order \cite{dalalyan2017theoretical}.

\textbf{Ex 3: Bounded perturbations.} We also admit potentials of the form
$f = f_{\text{sc}} + f_{\text{p}}$
such that $f_{\text{sc}}$ is $\growth_0$-strongly convex with 
$\lip_0$-Lipschitz gradient,
and $f_{\text{p}}$ satisfies
$|f_{\text{p}}|\vee \norm{\grad f_{\text{p}}}\vee \norm{\grad^2 f_{\text{p}}} \leq B$.
Then
\eqn{
    \inner{\grad f(x) - \grad f(y), x-y} 
    &=
    \inner{\grad f_{\text{sc}}(x) - \grad f_{\text{sc}}(y), x-y} +  \inner{\grad f_{\text{p}}(x) - \grad f_{\text{p}}(y),x-y}
    \\
    &\geq
    \growth_0 \norm{x-y}^2 - 2B\norm{x-y}
    \geq
    \growth \norm{x-y}^2 - \pert,
}
where the last inequality holds for $\growth=\growth_0/2$ and
$\pert = 2B^2/\growth$.
We also have $\|\Hess f\| \leq \|\Hess f_{\text{sc}}\|+\|\Hess
f_{\text{p}}\| \leq \lip_0+B\coloneqq \lip$.  Thus, Assumption~\ref{as:all}
holds for finite perturbations of strongly convex and smooth potentials.
Further, by the Holley-Stroock lemma~\cite{holley1987LSI}, \ref{eq:LSI} is
satisfied for $\poin = \growth_0^{-1}e^{2B}$.

\subsection{Strong convexity outside a ball}
For a second class of examples, consider potentials that are
strongly convex outside a ball. If we assume that $f$ has continuous and upper bounded Hessian $\nabla^2 f(x) \preceq M_0\id$,
this assumption means
\eqn{\label{eq:outside-ball}
    \inf_{\norm{x} \geq r} \Hess f(x)
    \succeq
    \growth_0 \id, \qquad \inf_{\norm{x} < r} \Hess f(x) \succeq -k \id,
}
where $\growth_0 > 0$ is the convexity parameter, and $r\geq 0$, $k>0$.
This condition implies $f$ is first-order smooth with $L\coloneqq M_0 \vee k$.
The first condition in \eqref{eq:outside-ball}
is enough to verify strong dissipativity; we write
$y = x + s_3 z$ for $s_3 = \norm{x-y}, \norm{z} = 1$,
and let $s_1, s_2 \in [0, s_3]$ such that $\norm{x + s z} \leq r \iff s \in [s_1, s_2]$.
Then, strong dissipativity follows from
\eqn{
    \inner{\grad f(x) - \grad f(y), x-y}
    &=
    \int_{s \in [0, s_1] \cup [s_2, s_3]} \hspace{-55pt} \inner{\Hess f(x + sz) z, s_3 z} ds + \inner{\grad f(x+s_2 z) - \grad f(x+s_1 z), s_3 z}
    \\
    &\geq
    \growth_0 (s_3 - 2 r)s_3 - 2\lip r s_3 \geq
    \tfrac{\growth_0}{2} \norm{x-y}^2 - \tfrac{2}{\growth_0}(\growth_0 r + \lip r)^2,
}
where outside of the ball we use the strong convexity, and inside we use
smoothness. 

For \ref{eq:LSI}, we replace $f$ with a strongly convex function 
$\tf(x) = f(x) + 0.5(k+\growth_0) \norm{x}^2 1_{\{ \norm{x} \leq r\}}$,
so that $\Hess \tf \succeq \growth_0 \id$ everywhere.
Then $\| f - \tf\|_\infty \leq 0.5 (k + m_0) r^2$,
and by the Holley-Stroock perturbation lemma \cite{holley1987LSI},
$
    \poin
    \leq
    \poin_{\tf} \cdot \exp\left((k+\growth_0)r^2\right)
    =m_0^{-1}
    \exp\big((k+\growth_0)r^2\big).
$

\textbf{Ex: Student's t-Regression with Gaussian prior.} Consider the function with
$\alpha > 0$
\eqn{\label{eq:quad-student}
    f(x) = \tfrac{1}{2} \log (1 + \norm{x}^2) + \tfrac{\alpha}{2} \norm{x}^2 + \text{constant}.
}
The gradient is $\frac{x}{1+\norm{x}^2} + \alpha x$, which is
$(\alpha +1)$-Lipschitz, and Hessian is
$\alpha \id + \frac{(1+\norm{x}^2)\id - 2xx^\top}{(1+\norm{x}^2)^2}.$
When $\alpha<1/8$ is sufficiently small, this function
is non-convex. However, if we take the radius to be
$\norm{x} \geq 1/\sqrt{\alpha}$;
we find strong convexity outside the ball with $\growth_0 = \alpha^2(\alpha+3)/(\alpha+1)^2$,
and $k = -1/8$. Strong dissipativity is
satisfied, and \ref{eq:LSI} holds with constant
$\tfrac{(\alpha+1)^2}{\alpha^2(\alpha+3)} \exp\big(\tfrac{\alpha^2(\alpha+3)}{\alpha(\alpha+1)^2}+\frac{1}{8\alpha}\big)$.

For instance, when data is
heavy tailed, Student's t-distribution is used to model the errors.  Under a
Gaussian prior on the coefficients $x\sim \Gsn(0, \alpha \id)$, the posterior
distribution has the potential
\eqn{
    f(x) = \sum_{i=1}^n \log(1+(y_i - \inner{v_i, x})^2) + \tfrac{\alpha}{2}\norm{x}^2 + \text{constant},
}
where $\{v_i\}_{i=1}^{n}, \{y_i\}_{i=1}^n$ are data samples as before. Notice
that the potential has the same form as \eqref{eq:quad-student}. Under
suitable assumption on data, one can use the same steps above to verify
our conditions.

\subsection{Non-uniform strong convexity}
Finally, we consider functions which are similar to the previous section, but
with variable convexity 
\eqn{
    \inf_{\norm{x} \geq r} \Hess f(x) \succeq \growth_0(r)\id.
}
Then if $\sup_{r \geq 0} \growth_0(r) > 0$, strong dissipativity holds for the
same reason as in the prior subsection, since we need only fix some $r > 0,$
where $\growth_0(r) > 0$ to recover the first inequality in
\eqref{eq:outside-ball}. \ref{eq:LSI} is satisfied as well \cite{chen1997estimates} with
the constant $\poin$ bounded by
\eqn{\textstyle
  \poin \leq \frac{a_0^2}{2} \exp \left(\int_0^{a_0} r \growth_0(r) dr - 1\right)
  \text{ where $a_0$ uniquely solves $\int_0^a \growth_0(r) = 2/a$.}
}

\textbf{Ex: Heavy-tailed regression with corrupted noise.} Consider the following function
\eqn{\label{eq:simp-blake} 
    f(x) = -\tfrac{1}{2} \log \left(\beta + \exp(-\norm{x}^2)\right) + \tfrac{\alpha}{2} \norm{x}^2 + \text{constant},
}
where $\beta > 0, \alpha > 0$. The gradient is
$\frac{x}{\beta \exp(\norm{x}^2) + 1} + \alpha x$,
which is $\alpha + \tfrac{1}{1+\beta}$-Lipschitz, and the Hessian is
$\tfrac{\beta \exp(\norm{x}^2) (\id-2x x^\top) + \id}{(\beta \exp(\norm{x}^2) + 1)^2} +\alpha \id$.
In this case,
$\growth_0(r) = \alpha - \frac{\beta \exp(r^2) (2r^2 - 1) -1}{(\beta \exp(r^2) + 1)^2} \geq \alpha - \frac{r^2}{\beta \exp{(r^2)}}$.
Since $\frac{r^2}{\beta \exp( r^2)} \to 0$, this quantity eventually becomes
positive. By our argumentation, the function is strongly dissipative and we can
solve for the \ref{eq:LSI} constant through numerical integration.

For an instance of this, consider linear
regression on data $\{v_i\}_{i=1}^n, \{y_i\}_{i=1}^n$, with corrupted Gaussian
noise such that with small probability the noise is instead sampled from some
uniform distribution (this arises in visual reconstruction
problems~\cite{blake1987visual}). Assuming a prior $x\sim\Gsn(0, \alpha \id)$,
we obtain the following potential for the posterior distribution
\eqn{
    f(x) = -\frac{1}{2} \sum_{i=1}^n \log \left(\beta + \exp(-(y_i - \inner{v_i, x})^2)\right) + \frac{\alpha}{2} \norm{x}^2 + \text{constant}.
}
This potential has the same structure as \eqref{eq:simp-blake} and our
assumptions can be verified using the same steps, under suitable conditions on
the data.

%
\section{Conclusion}\label{sec:con}

In this paper, we analyzed the convergence of unadjusted LMC algorithm for a
class of potentials that are first-order smooth and satisfy strong
dissipativity.  We used the Fokker-Planck equation of the interpolated process
alongside the smoothness assumptions to obtain a differential inequality which
in turn yielded a single step bound to be iterated to obtain our main
convergence results. In the case of strongly convex potentials, the obtained
rates improve upon the existing rates in Chi-squared and \Renyi divergence, and recover the
state-of-the-art rates in KL, TV, as well as $\*{W}_2$.

There are several important future directions to consider, among which we
highlight a few here.
\begin{itemize}[noitemsep,leftmargin=.3in]
\item Although we assumed \ref{eq:LSI} throughout the paper, \poincare
    inequality is sufficient to establish the differential inequality in
    Lemma~\ref{lem:disc-fok-plan-1}. The restriction is due to the strong
    dissipativity assumption, which enforces a quadratic growth on the
    potential function; thus, \ref{eq:LSI} must hold.  This also enforces at least
    linear growth on the gradient, which prevents us from considering weakly
    smooth potentials that satisfy \Holder continuity.  Relaxing this
    assumption to weak dissipativity may allow one to rely on \poincare
    inequality, which permits weakly smooth potentials
    that have at least linear growth. A promising set of conditions in this context is, for
    some $\theta<1$ and $\gamma \in[1,2)$
  \eqn{
    & \norm{\grad f(x)- \grad f(y)}
    \leq
    \lip \norm{x-y}^\theta,
    & \inner{x, \grad f(x) } 
    \geq
    \growth \norm{x}^\gamma - \pert.
  }
  We note that this setting is already considered in \cite{erdogdu2020convergence}
  under the KL divergence.
\item We mentioned that the rate estimates presented in this paper do not hold
  for the subsequent iterates, similar to \cite{shen2019randomized,ganesh2020faster,erdogdu2020convergence}.
This is
    an artifact of the proof technique we rely on, and hopefully can be
    remedied in the future work.
\item Working out a simple Gaussian toy example, one can verify that the LMC
    algorithm converges to the target in $\tilde{\mathcal{O}}(\sqrt{d/\eps})$
    iterations in Chi-squared divergence. This is $\tilde{\mathcal{O}}(\sqrt{d/\eps})$ better than the rate
    estimate we obtained in the current paper, and is due to
the additional second-order smoothness of the Gaussian potential. Therefore, achieving
    this rate under second-order smoothness is an interesting direction left for future work.
    \item Similar techniques can be utilized to establish convergence rates for other numerical schemes;
    the ones that satisfy certain optimality criteria are of particular interest~\cite{shen2019randomized,cao2020complexity}.
\end{itemize}

Finally, we note that many of the bounds in the paper can be improved, at the
expense of introducing some additional complexity into the results.

%
\section*{Acknowledgements}
This research is partially funded by NSERC Grant [2019-06167], Connaught New
Researcher Award, and CIFAR AI Chairs program at the Vector Institute.  
RH and MAE thank Andre Wibisono for pointing out an error in a previous version of
this paper in Lemma 1. To fix this, stronger assumptions were made in the current version.
Finally, 
RH is grateful to Arun Ganesh for the helpful email exchange on formalizing Lemma~\ref{lem:convergence}.
\bibliographystyle{amsalpha}
{\small \bibliography{./bib}}

\newcommand{\etalchar}[1]{$^{#1}$}
\providecommand{\bysame}{\leavevmode\hbox to3em{\hrulefill}\thinspace}
\providecommand{\MR}{\relax\ifhmode\unskip\space\fi MR }
\providecommand{\MRhref}[2]{%
  \href{http://www.ams.org/mathscinet-getitem?mr=#1}{#2}
}
\providecommand{\href}[2]{#2}
\begin{thebibliography}{CCAY{\etalchar{+}}18}

\bibitem[BBCG08]{bakry2008simple}
Dominique Bakry, Franck Barthe, Patrick Cattiaux, and Arnaud Guillin, \emph{A
  simple proof of the {P}oincar{\'e} inequality for a large class of
  probability measures}, Electron. Commun. Probab. \textbf{13} (2008), 60--66.

\bibitem[BDMS19]{brosse2019tamed}
Nicolas Brosse, Alain Durmus, Eric Moulines, and Sotirios Sabanis, \emph{{T}he
  tamed unadjusted {L}angevin algorithm}, Stochastic Processes and their
  Applications \textbf{129} (2019), no.~10, 3638 -- 3663.

\bibitem[B{\'E}85]{bakry1985LSI}
D.~Bakry and M.~{\'E}mery, \emph{{D}iffusions hypercontractives},
  {S}{\'e}minaire de {P}robabilit{\'e}s {X}{I}{X} 1983/84 (Berlin, Heidelberg),
  Springer Berlin Heidelberg, 1985, pp.~177--206.

\bibitem[Bel43]{bellman1943stability}
Richard Bellman, \emph{{T}he stability of solutions of linear differential
  equations}, Duke Math. J. \textbf{10} (1943), no.~4, 643--647.

\bibitem[BG99]{bobkov1999exponential}
Sergej~G Bobkov and Friedrich G{\"o}tze, \emph{{E}xponential integrability and
  transportation cost related to logarithmic {S}obolev inequalities}, Journal
  of Functional Analysis \textbf{163} (1999), no.~1, 1--28.

\bibitem[BV05]{bolley2005weighted}
Fran{\c{c}}ois Bolley and C{\'e}dric Villani, \emph{{W}eighted
  {C}sisz{\'a}r-{K}ullback-{P}insker inequalities and applications to
  transportation inequalities}, {A}nnales de la {F}acult{\'e} des sciences de
  {T}oulouse: {M}ath{\'e}matiques, vol.~14, 2005, pp.~331--352.

\bibitem[BZ87]{blake1987visual}
Andrew Blake and Andrew Zisserman, \emph{Visual reconstruction}, 1987.

\bibitem[CB18]{cheng2018convergence}
Xiang Cheng and Peter~L Bartlett, \emph{{C}onvergence of {L}angevin
  {M}{C}{M}{C} in {K}{L}-divergence}, PMLR 83 (2018), no.~83, 186--211.

\bibitem[CCAY{\etalchar{+}}18]{cheng2018sharp}
Xiang Cheng, Niladri~S Chatterji, Yasin Abbasi-Yadkori, Peter~L Bartlett, and
  Michael~I Jordan, \emph{{S}harp convergence rates for {L}angevin dynamics in
  the nonconvex setting}, arXiv preprint arXiv:1805.01648 (2018).

\bibitem[CGL{\etalchar{+}}20]{chewi2020exponential}
Sinho Chewi, Thibaut~Le Gouic, Chen Lu, Tyler Maunu, Philippe Rigollet, and
  Austin Stromme, \emph{{E}xponential ergodicity of mirror-{L}angevin
  diffusions}, arXiv preprint arXiv:2005.09669 (2020).

\bibitem[CGW10]{cattiaux2010note}
Patrick Cattiaux, Arnaud Guillin, and Li-Ming Wu, \emph{{A} note on
  {T}alagrand's transportation inequality and logarithmic {S}obolev
  inequality}, Probability theory and related fields \textbf{148} (2010),
  no.~1-2, 285--304.

\bibitem[CLL19]{cao2019exponential}
Yu~Cao, Jianfeng Lu, and Yulong Lu, \emph{Exponential decay of r{\'e}nyi
  divergence under fokker--planck equations}, Journal of Statistical Physics
  \textbf{176} (2019), no.~5, 1172--1184.

\bibitem[CLW20]{cao2020complexity}
Yu~Cao, Jianfeng Lu, and Lihan Wang, \emph{Complexity of randomized algorithms
  for underdamped langevin dynamics}, arXiv preprint arXiv:2003.09906 (2020).

\bibitem[CW97]{chen1997estimates}
Mu-Fa Chen and Feng-Yu Wang, \emph{Estimates of logarithmic sobolev constant:
  An improvement of bakry–emery criterion}, Journal of Functional Analysis
  \textbf{144} (1997), no.~2, 287--300.

\bibitem[Dal17a]{dalalyan2017furthur}
Arnak Dalalyan, \emph{{F}urther and stronger analogy between sampling and
  optimization: {L}angevin {M}onte {C}arlo and gradient descent}, {P}roceedings
  of the 2017 {C}onference on {L}earning {T}heory, Proceedings of Machine
  Learning Research, vol.~65, PMLR, 07--10 Jul 2017, pp.~678--689.

\bibitem[Dal17b]{dalalyan2017theoretical}
Arnak~S Dalalyan, \emph{{T}heoretical guarantees for approximate sampling from
  smooth and log-concave densities}, Journal of the Royal Statistical Society:
  Series B (Statistical Methodology) \textbf{79} (2017), no.~3, 651--676.

\bibitem[DK19]{dalalyan2019user}
Arnak~S Dalalyan and Avetik Karagulyan, \emph{{U}ser-friendly guarantees for
  the {L}angevin {M}onte {C}arlo with inaccurate gradient}, Stochastic
  Processes and their Applications \textbf{129} (2019), no.~12, 5278--5311.

\bibitem[DM16]{durmus2016sampling}
Alain Durmus and Eric Moulines, \emph{{S}ampling from strongly log-concave
  distributions with the {U}nadjusted {L}angevin {A}lgorithm}, arXiv preprint
  arXiv:1605.01559 \textbf{5} (2016).

\bibitem[DM17]{durmus2017nonasymptotic}
\bysame, \emph{Nonasymptotic convergence analysis for the unadjusted langevin
  algorithm}, Annals of Applied Probability \textbf{27} (2017), no.~3,
  1551--1587.

\bibitem[DM19]{durmus2019high}
\bysame, \emph{{H}igh-dimensional {B}ayesian inference via the unadjusted
  {L}angevin algorithm}, Bernoulli \textbf{25} (2019), no.~4A, 2854--2882.

\bibitem[DMM19]{durmus2019analysis}
Alain Durmus, Szymon Majewski, and Blazej Miasojedow, \emph{{A}nalysis of
  {L}angevin {M}onte {C}arlo via {C}onvex {O}ptimization.}, Journal of Machine
  Learning Research \textbf{20} (2019), no.~73, 1--46.

\bibitem[EH20]{erdogdu2020convergence}
Murat~A Erdogdu and Rasa Hosseinzadeh, \emph{{O}n the convergence of {L}angevin
  {M}onte {C}arlo: {T}he interplay between tail growth and smoothness}, arXiv
  preprint arXiv:2005.13097 (2020).

\bibitem[EMS18]{erdogdu2018global}
Murat~A Erdogdu, Lester Mackey, and Ohad Shamir, \emph{{G}lobal non-convex
  optimization with discretized diffusions}, {A}dvances in {N}eural
  {I}nformation {P}rocessing {S}ystems, 2018, pp.~9671--9680.

\bibitem[GT20]{ganesh2020faster}
Arun Ganesh and Kunal Talwar, \emph{Faster differentially private samplers via
  {R}{\'e}nyi divergence analysis of discretized {L}angevin {M}{C}{M}{C}},
  Advances in Neural Information Processing Systems \textbf{33} (2020).

\bibitem[HBE20]{he2020ergodicity}
Ye~He, Krishnakumar Balasubramanian, and Murat~A Erdogdu, \emph{On the
  ergodicity, bias and asymptotic normality of randomized midpoint sampling
  method}, Advances in Neural Information Processing Systems \textbf{33}
  (2020).

\bibitem[HS87]{holley1987LSI}
Richard Holley and Daniel Stroock, \emph{Logarithmic {S}obolev inequalities and
  stochastic {I}sing models}, Journal of Statistical Physics \textbf{46}
  (1987), no.~5, 1159--1194.

\bibitem[IW14]{ikeda2014stochastic}
Nobuyuki Ikeda and Shinzo Watanabe, \emph{{S}tochastic differential equations
  and diffusion processes}, Elsevier, 2014.

\bibitem[Liu20]{liu2020poincare}
Yuan Liu, \emph{The poincar{\'e} inequality and quadratic
  transportation-variance inequalities}, Electron. J. Probab. \textbf{25}
  (2020), 16 pp.

\bibitem[LW20]{lu2020complexity}
Jianfeng Lu and Lihan Wang, \emph{Complexity of zigzag sampling algorithm for
  strongly log-concave distributions}, arXiv preprint arXiv:2012.11094 (2020).

\bibitem[LWME19]{li2019stochastic}
Xuechen Li, Yi~Wu, Lester Mackey, and Murat~A Erdogdu, \emph{stochastic
  {R}unge-{K}utta accelerates {L}angevin {M}onte {C}arlo and beyond},
  {A}dvances in {N}eural {I}nformation {P}rocessing {S}ystems 32, Curran
  Associates, Inc., 2019, pp.~7748--7760.

\bibitem[MFWB19]{mou2019improved}
Wenlong Mou, Nicolas Flammarion, Martin~J Wainwright, and Peter~L Bartlett,
  \emph{Improved bounds for discretization of {L}angevin diffusions:
  Near-optimal rates without convexity}, arXiv preprint arXiv:1907.11331
  (2019).

\bibitem[Mir17]{mironov2017renyi}
Ilya Mironov, \emph{{R}{\'e}nyi differential privacy}, 2017 {I}{E}{E}{E} 30th
  {C}omputer {S}ecurity {F}oundations {S}ymposium ({C}{S}{F}), IEEE, 2017,
  pp.~263--275.

\bibitem[MSH02]{mattingly2002ergodicity}
Jonathan~C Mattingly, Andrew~M Stuart, and Desmond~J Higham, \emph{{E}rgodicity
  for {S}{D}{E}s and approximations: locally {L}ipschitz vector fields and
  degenerate noise}, Stochastic processes and their applications \textbf{101}
  (2002), no.~2, 185--232.

\bibitem[Oks13]{oksendal2013stochastic}
Bernt Oksendal, \emph{{S}tochastic differential equations: an introduction with
  applications}, Springer Science \& Business Media, 2013.

\bibitem[Ris96]{risken1996fokker}
Hannes Risken, \emph{{F}okker-{P}lanck equation}, Springer, 1996.

\bibitem[RRT17]{raginsky2017non}
Maxim Raginsky, Alexander Rakhlin, and Matus Telgarsky, \emph{{N}on-convex
  learning via stochastic gradient {L}angevin dynamics: a nonasymptotic
  analysis}, {P}roceedings of the 2017 {C}onference on {L}earning {T}heory,
  vol.~65, 2017, pp.~1674--1703.

\bibitem[SL19]{shen2019randomized}
Ruoqi Shen and Yin~Tat Lee, \emph{{T}he randomized midpoint method for
  log-concave sampling}, {A}dvances in {N}eural {I}nformation {P}rocessing
  {S}ystems, 2019, pp.~2098--2109.

\bibitem[Tsy08]{tsybakov2008introduction}
Alexandre~B Tsybakov, \emph{{I}ntroduction to nonparametric estimation},
  Springer Science \& Business Media, 2008.

\bibitem[VEH14]{van2014renyi}
Tim Van~Erven and Peter Harremos, \emph{{R}{\'e}nyi divergence and
  {K}ullback-{L}eibler divergence}, IEEE Transactions on Information Theory
  \textbf{60} (2014), no.~7, 3797--3820.

\bibitem[Vil03]{villani2003topics}
C{\'e}dric Villani, \emph{{T}opics in optimal transportation}, no.~58, American
  Mathematical Soc., 2003.

\bibitem[VW19]{vempala2019rapid}
Santosh Vempala and Andre Wibisono, \emph{Rapid convergence of the unadjusted
  {L}angevin algorithm: {I}soperimetry suffices}, {A}dvances in {N}eural
  {I}nformation {P}rocessing {S}ystems, 2019, pp.~8092--8104.

\bibitem[YBVE20]{yu2020analysis}
Lu~Yu, Krishnakumar Balasubramanian, Stanislav Volgushev, and Murat~A Erdogdu,
  \emph{{A}n analysis of constant step size sgd in the non-convex regime:
  {A}symptotic normality and bias}, arXiv preprint arXiv:2006.07904 (2020).

\end{thebibliography}
\appendix
\section{Proofs of the Main Results}
\label{sec:proof-main}
\begin{proof-of-lemma}[\ref{lem:uniform-warmness}]
  We use Cauchy-Schwarz inequality to get
  \eqn{
    \Esub{\frac{\trho_T}{\target}(x)^{2\alpha-2}}{\trho_T} \leq
    \Esub{\frac{\pi_T}{\target}(x)^{4\alpha-3}}{\target}^{\frac{1}{2}} 
    \Esub{\frac{\trho_T}{\pi_T}(x)^{4\alpha-2}}{\pi_T}^{\frac{1}{2}}.
  }
  
  For the first term, we then apply Lemmas \ref{lem:renyi-cont-decay} and \ref{lem:init}, and obtain that when 
  \eqn{N\step \geq \frac{(4\alpha-3)\poin}{2} \log \log \Esub{\frac{\rho_0}{\target}(x)^{4\alpha-3}}{\target} ,
  }
  we have $\Esub{\frac{\pi_T}{\target}(x)^{4\alpha-3}}{\target}  \leq e$ for $T \geq N\step$. We can apply a crude bound of $\Esub{\frac{\rho_0}{\target}(x)^{4\alpha-3}}{\target}  \leq \exp{(4\alpha dC_{\sigma})}$, where $C_{\sigma}$ is as in the remark after Lemma~\ref{lem:init}.

  For the second term, we apply Lemma \ref{lem:disc-cont-div} with $16\alpha - 10$ and number of iterations $2N$, under the condition that 
  \eqn{
    c \alpha^2\cond^2 \lip^2 (\pert + d+\log 2N)N\step^2 \leq 1.
  }
  Then we get $\Esub{\frac{\trho_T}{\pi_T}(x)^{4\alpha-2}}{\pi_T} \leq 20e\sqrt{\alpha}$; combining these two terms yields the final bound. 
\end{proof-of-lemma}
\begin{proof-of-lemma}[\ref{lem:sing-step}]
Suppose $x^*$ is the global minimizer of $f$ (therefore $\grad f(x^*)=0$).
From Lemma~\ref{lem:semi-contract} and $\step \leq \growth/\lip^2$ we have
\eqn{
    \norm{x - \step \grad f(x) - x^*}^2 \leq (1 - \growth\step)\norm{x - x^*}^2 + 2\pert\step.
}
Using the previous inequality with the fact that Gaussian has zero odd moments
and Lemma~\ref{lem:rec-bound} we get
\eqn{
    \EE{\norm{x_k -x^*}^2} \leq
    \EE{\norm{x_0 -x^*}^2}  + \frac{2(\pert+d)}{\growth}.
}
Doing the same for power $4$ we get the following
\eqn{
    \EE{\norm{x_{k+1}-x^*}^4}
    \leq
    (1-\growth\step)\EE{\norm{x_k-x^*}^4}
    +12 \step (\pert+d) \EE{\norm{x_k-x^*}^2} + 12 \step^2 (\pert+d)^2.
}
Plugging the bound on $\EE{\norm{x_k-x^*}^2}$ back in the previous inequality
and using Lemma~\ref{lem:rec-bound}, we get the following
\eqn{
    \EE{\norm{x_k-x^*}^4}
    \leq
    \left[
        \frac{\EE{\norm{x_0-x^*}^4}}{(\pert+d)^2} + \frac{12}{\growth}\left(
        \frac{\EE{\norm{x_0-x^*}^2}}{\pert+d} + \frac{2}{\growth} + \step
    \right)
    \right]
    (\pert+d)^2.
}
For $t\leq \step$, we write
\eqn{
    \EE{\norm{t \grad f(x_k) + \sqrt{2} B_t}^4}
    & \leq
    8 \step^4 \EE{\norm{ \grad f(x_k)}^4} + 32 \EE{\norm{B_t}}^4
    \\
    & \leq
    8 \step^4 \lip^4 \EE{\norm{x_k - x^*}^4} + 96 \step^2 d^2.
}
By combining the last two inequalities we get the following.
\eqn{
    &\EE{\norm{t \grad f(x_k) + \sqrt{2} B_t}^4}\\
    & \leq
    \left(
    96 + 8\step^2\lip^4
    \left[
        \frac{\EE{\norm{x_0-x^*}^4}}{(\pert+d)^2} +
        \frac{12\EE{\norm{x_0-x^*}^2}}{\growth(\pert+d)} +
        \frac{24}{\growth^2} +
        \frac{12 \step}{\growth}
    \right]
    \right)
    (\pert+d)^2 \step^2,
}
using the definition of $\beta$ and moments of Gaussian and the bound on
$\step$ we get
\eqn{\label{eq:second-term-diff-eq}
    \EE{\norm{t \grad f(x_k) + \sqrt{2} B_t}^4}^{1/2}
    \leq
    c \beta(\pert+d)\step,
}
for some universal constant $c$.  Now if the conditions of
Lemma~\ref{lem:uniform-warmness} hold, then we have the expected ratio of the
densities bounded by an absolute constant for $N\leq k\step+t\leq2N$. Combining
this with the previous bound, we obtain
\eqn{
  \EE{\frac{\trho_{k\step+t}}{\target}\left( \tx_{k\step+t}\right)^2}^{1/2}
     \EE{\norm{\grad f(\tx_{k\step+t}) - \grad f(x_k)}^4}^{1/2}
    &\leq
    c\beta \lip^2 (\pert + d) \step
}
where $c$ is an absolute constant.

We plug in the derived upper bounds back in $\eqref{eq:disc-fok-plan-1}$ to get
\eqn{
    \frac{d}{d t} \chisq{\trho_{k\step+t}}{\target} 
    \leq
    -\frac{3}{2\poin} \chisq{\trho_{k\step+t}}{\target}
    + c \beta \lip^2 (\pert + d) \step.
}
Integrating this differential inequality and using $t \leq \step$ results in the
following single step bound
\eqn{
    \chisq{\rho_{k+1}}{\target} 
    \leq
    \Big(1-\frac{3\step}{4\poin}\Big)\chisq{\rho_{k}}{\target} + c \beta \lip^2 (\pert + d) \step^2,
}
where we used $\step \leq \frac{\poin}{2}$ and that $e^{-x}\leq 1-x/2$ for $x\in[0,1]$, and as before we absorbed universal
constants into $c$.
\end{proof-of-lemma}
\begin{proof-of-theorem}[\ref{thm:main}]
In the first phase,  by Lemma~\ref{lem:uniform-warmness}, we used the first $N$
steps of LMC to show that the expected squared ratio of densities evaluated at
the interpolation process is bounded by an absolute constant and it stays
bounded for the subsequent $N$ iterations.  This allows us to iterate the single-step bound
provided by Lemma~\ref{lem:sing-step}, for which we need to bound its
starting point, the LMC iteration $N$.  We write
\eqn{    \chisq{\rho_{N}}{\target}  =
  \Esub{\frac{\rho_{N}}{\target}\left(x\right)^2}{\target}-1
\leq
\Esub{\frac{\rho_{N}}{\pi_{N\step}}\left(x\right)^4}{\pi_T}^{1/2}
\Esub{\frac{\pi_{N\step}}{\target}\left(x\right)^3}{\target}^{1/2} -1,
}
by Cauchy-Schwartz inequality. We bound the right hand side term by term with
an absolute constant. For the first term we use Lemma~\ref{lem:disc-cont-div} for $\alpha=14$.
In order for it to be bounded by $22e$, it is sufficient if the following holds
that
$
    c_1 \cond^2 \lip^2 N (\pert + d+ \log{N}) \step^2 \leq 1,
$
where $c_1$ is some universal constant.  The second term can be written as
$$\Esub{\frac{\pi_{N\step}}{\target}\left(x\right)^3}{\target}= \Esub{\frac{\pi_{N\step}}{\target}\left(x\right)^{4\alpha-3}}{\target}  \text{ for $\alpha=3/2$},$$
for which we already obtained an upper bound in the proof of
Lemma~\ref{lem:uniform-warmness}. That is, whenever
$N\step \geq \frac{3\poin}{2} \log(6d C_\sigma)$,
the right hand side is bounded by $e$.
Thus, we have 
$
\chisq{\rho_{N}}{\target} \leq 12.
$

In the second phase of the analysis, we use the differential
inequality.  Iterating the single step bound in Lemma~\ref{lem:sing-step} with
the help of Lemma~\ref{lem:rec-bound}, together with the upper bound on the
initialization of the second phase, we get
\eqn{
    \chisq{\rho_{2N}}{\target} 
    \leq
    \exp{\Big(\!-\frac{3\step N}{4\poin}\Big)} 12 + c\poin \lip^2 \beta (\pert + d) \step,
}
where we used the bound on Chi-square divergence in step $N$ and absorbed
absolute constants into $c$.  To make this less than $\eps$, it suffices if the
following additional inequalities hold
\eqn{
    N \step  \geq \frac{4}{3} \poin \log \frac{24}{\eps},
    \ \ \ \ \ \ \ \ \ \ 
    \step \leq \frac{1}{2c\beta\poin \lip^2} \times \frac{\eps}{\pert + d}
}
for some absolute constant $c$.
Combining the above conditions with the conditions of used Lemmas,
and simplifying the statements, we obtain the inequalities in \eqref{eq:main-ineqs}.

For the last statement, it suffices to check the inequalities in \eqref{eq:main-ineqs}
for the given choice of step size and the number of iterations.
\end{proof-of-theorem}
\subsection{Proofs for the \Renyi Divergence}
\label{sec:renyi-proofs}
\begin{proof-of-lemma}[\ref{lem:renyi-diff-ineq}]
  Define the following quantities for $\alpha \geq 1$ between two densities $p,q$
  \eqn{\label{eq:renyi-diff-ineq-2}
    \finformation{\alpha}{p}{q} = \Esub{\frac{p}{q}(x)^{\alpha}}{q},
    \qquad
    \ginformation{\alpha}{p}{q} = \Esub{\frac{p}{q}(x)^{\alpha-2} \norm{\grad \frac{p}{q}(x)}^2}{q}.
  }
  Then we have $(\alpha-1)\renyi{\alpha}{p}{q} = \log \finformation{\alpha}{p}{q}$.
  Note that $\finformation{\alpha}{p}{q} \geq 1$ for any
  $p,q$. We have the following result for the Langevin diffusion.
  \begin{lemma}\label{lem:renyi-ratio-bound}
    (Adapted from \cite[Lemma~5]{vempala2019rapid})
    If $q$ satisfies LSI with constant $\poin$, then
    $$
    \frac{ \ginformation{\alpha}{p}{q}}{\finformation{\alpha}{p}{q}}
    \geq
    \frac{2}{\alpha^2 \poin} \renyi{\alpha}{p}{q}.
    $$
  \end{lemma}
The above lemma was used in \cite{vempala2019rapid} to prove the exponential convergence of the Langevin diffusion.
We use it in a similar way, but for the interpolation process.

Consider the dynamics in \eqref{eq:time-derivative-trho} and write
to express the time-derivative of $F_\alpha$ as
\eqn{\label{eq:renyi-diff-ineq-3}
    \frac{d}{dt} \finformation{\alpha}{\trho_{t}}{\target}
    &=
    \alpha \int \frac{\trho_t}{\target}(x)^{\alpha-1}\frac{\partial \trho_{t}(x)}{\partial t} dx \\
    &=
    \alpha \int \frac{\trho_t}{\target}(x)^{\alpha-1}
    \grad \cdot
    \left(
      \trho_{t}(x) 
      \left(
        \EE{\grad f(x_k) - \grad f(x)\vert \tx_{t}=x} + 
        \grad \log\frac{\trho_t}{\target}(x)
      \right)
    \right)dx
\\
    &\stackrel{1}{=}
    \alpha(\alpha-1) \int \frac{\trho_t}{\target}(x)^{\alpha-2}
    \inner{\grad \frac{\trho_t}{\target}(x),
      \EE{\grad f(x) - \grad f(x_k)\vert \tx_{t}=x} - \grad \log\frac{\trho_t}{\target}(x)
  } \trho_{t}(x) dx,
}
where in $1$ we use the divergence theorem.
For the first term, we write
\eqn{
    &\int \frac{\trho_t}{\target}(x)^{\alpha-2}
    \inner{ \grad \frac{\trho_t}{\target}(x) \sqrt{\frac{\target(x)}{\trho_t(x)}},
    \sqrt{\frac{\trho_t(x)}{\target(x)}} \EE{\grad f(x) - \grad f(x_k)\vert \tx_{t}=x}} \trho_{t}(x) dx
    \\
    &\stackrel{1}{\leq}
    \frac{1}{4}\int \frac{\trho_t}{\target}(x)^{\alpha-2}
    \norm{\grad\frac{\trho_t}{\target}(x)}^2 \target(x)dx
    \\
    & \qquad +
    \int \EE{\frac{\trho_t}{\target}(x)^{\alpha-1} \norm{\grad f(x) - \grad f(x_k)}^2\vert \tx_{t}=x}\trho_t(x)dx
    \\
    &=
    \frac{1}{4} \ginformation{\alpha}{\trho_t}{\target} +
    \EE{\frac{\trho_t}{\target}(\tx_t)^{\alpha-1}\norm{\grad f(\tx_t) - \grad f(x_k)}^2},
}
where in $1$ we used that $\inner{a,b} \leq \frac{1}{4} \norm{a}^2 + \norm{b}^2$.
For the second term, we find
\eqn{
    &\int
    \frac{\trho_t}{\target}(x)^{\alpha-2}
    \inner{\grad\frac{\trho_t}{\target}(x),
    -\grad \log\frac{\trho_t}{\target}(x)} \trho_{t}(x) dx
    =
    - \ginformation{\alpha}{\trho_t}{\target}
}
Combining terms, we get
\eqn{\label{eq:renyi-diff-ineq-3}
        \frac{d}{dt} \renyi{\alpha}{\trho_{t}}{\target}
        &=
        \frac{1}{(\alpha-1)\finformation{\alpha}{\trho_{t}}{\target}} \frac{d\finformation{\alpha}{\trho_{t}}{\target}}{dt}
        \\
        &\leq
        -\frac{3\alpha}{4} \frac{\ginformation{\alpha}{\trho_t}{\target}}{\finformation{\alpha}{\trho_{t}}{\target}}
            +
        \alpha \frac{\EE{\frac{\trho_t}{\target}(\tx_t)^{\alpha-1}\norm{\grad f(\tx_t) - \grad f(x_k)}^2}}{\finformation{\alpha}{\trho_{t}}{\target}}
        \\
        &\stackrel{1}{\leq}
        -\frac{3\alpha}{4} \frac{ \ginformation{\alpha}{\trho_t}{\target}}{\finformation{\alpha}{\trho_{t}}{\target}}
            +
        \alpha \EE{\frac{\trho_t}{\target}(\tx_t)^{2\alpha-2}}^{\frac{1}{2}} \EE{\norm{\grad f(\tx_t) - \grad f(x_k)}^4}^{\frac{1}{2}},
}
where in $1$ we use a Cauchy-Schwarz inequality and that
$\finformation{\alpha}{\trho_t}{\target} \geq 1$. It remains to apply Lemma
\ref{lem:renyi-ratio-bound} on the first term.
\end{proof-of-lemma}

\begin{proof-of-lemma}[\ref{lem:renyi-final-bound}]
    The bound on the second term in \eqref{eq:sing-step-renyi} is obtained directly from \eqref{eq:second-term-diff-eq}, as
    \eqn{
    \EE{\norm{\nabla f(\tx_{k\step + t}) - \nabla f(x_k)}^4}^{\frac{1}{2}} \leq c\beta \lip^2 (\pert + d)\step.
    }
    So combining this with Lemma \ref{lem:uniform-warmness}, for any $k \in [N, 2N]$, we get
    \eqn{
   \EE{\frac{\trho_{k\eta +t}}{\target}(\tx_{k\step + t})^{2\alpha-2}}^{\frac{1}{2}}\EE{\norm{\nabla f(\tx_{k\step + t}) - \nabla f(x_k)}^4}^{\frac{1}{2}} \leq c\beta \lip^2 (\pert + d)\alpha^{1/8}\step.
    }
    Substitution into Lemma \ref{lem:renyi-diff-ineq} yields
    \eqn{
    \frac{d}{dt} \renyi{\alpha}{\trho_t}{\target} \leq -\frac{3}{2\alpha\poin} \renyi{\alpha}{\trho_t}{\target} + c\beta \lip^2 (\pert + d)\alpha^{9/8} \step.
    }
    It remains to integrate this for $t \leq \step$, and apply $e^{-x} \leq (1-\frac{x}{2})$ for $x \in [0,1]$, and $\step \leq \frac{2\alpha \poin}{3}$.
\end{proof-of-lemma}

\begin{proof-of-theorem}[\ref{thm:main-renyi}]
    In the first phase, in accordance with the proof of Theorem \ref{thm:main}, we will need to ensure that the density ratio at step $N$ remains bounded. So we again write
    \eqn{\Esub{\frac{\trho_{N\eta}}{\target}(x)^\alpha}{\target} \leq \Esub{\frac{\trho_{N\eta}}{\pi_{N\step}}(x)^{2\alpha}}{\pi_{N\step}}^{\frac{1}{2}} \Esub{\frac{\pi_{N\eta}}{\target}(x)^{2\alpha-1}}{\target}^{\frac{1}{2}} .
    }
    For the first term, we can use Lemma \ref{lem:disc-cont-div} with $8\alpha-2$, so that again we have the condition
    \eqn{
       c \alpha^2 \cond^2 \lip^2 (\pert + d+\log N)N\step^2 \leq 1,
    }
    to guarantee that the first term is bounded by $7\alpha^{0.25}$. The second term is bounded via a \Holder inequality
    \eqn{
     \Esub{\frac{\pi_{N\step}}{\target}(x)^{2\alpha-1}}{\target} \leq \Esub{\frac{\pi_{N\step}}{\target}(x)^{4\alpha-3}}{\target}^{\frac{2\alpha-1}{4\alpha-3}},
    }
    which is simply bounded by $\exp(\frac{2\alpha-1}{4\alpha-3})$ under the conditions of Lemma \ref{lem:uniform-warmness}. Consequently the \Renyi divergence at $k=N$ is bounded by
    \eqn{
        \renyi{\alpha}{\rho_N}{\target} \leq \frac{\log(12\alpha)}{\alpha-1}.
    }
In the second phase, we simply iterate the differential inequality in Lemma \ref{lem:renyi-final-bound} and apply Lemma \ref{lem:gronwall} to get
    \eqn{
      \renyi{\alpha}{\rho_{2N}}{\target} \leq \exp \left(-\frac{3N\step}{4\alpha\poin} \right)\frac{\log(12\alpha)}{\alpha-1} + c\beta \poin \lip^2 (\pert + d)\alpha^{17/8} \step.
    }
    Thus we obtain $\eps$ accuracy if the following inequalities hold
    
\eqn{
    N \step  \geq \frac{4}{3} \alpha \lambda \log \frac{2\log(12\alpha)}{(\alpha-1)\eps},
    \ \ \ \ \ \ \ \ \ \ 
    \step \leq \frac{1 }{2c\beta \lambda \lip^2 \alpha^{17/8}} \times \frac{\eps}{\pert + d}
}
for some absolute constant $c$.
Combining these conditions with the conditions of other lemmas used above
and simplifying the expressions using $\alpha, C_\sigma > 1 \geq \eps$,
we conclude the proof of the first part for the given choice of $\eta$ and $N$,
and choosing a sufficiently small $\eps$.

The second part of the theorem follows from verifying the conditions
in the first part, and choosing a suitably small accuracy $\epsilon$. 
\end{proof-of-theorem}

%
\section{Main Technical Lemmas}
Let $X_{t}$ show the clipped interpolation process, with step size $\step$,
which is defined similar to interpolation process with the following exception:
if for any $s \leq t$ we have $\norm{X_{s} - X_{\floor{s/\step}\step}} > r$ we
change $X_{t}$ to $\perp$. We define $X^j_{t}$ similarly, with step size
$\step/j$ for $j \in \naturals$ and define $X'_{t}$ similarly for the
continuous time process. We use the same $r$ and $\step$ for all processes, in
other words we change $X^j_t$ and $X'_t$ to $\perp$ when for some $s \leq t$ we
have $\norm{X^j_{s} - X^j_{\floor{s/\step}\step}} > r$ and
$\norm{X'_{s} - X'_{\floor{s/\step}\step}} > r,$ respectively.
We will refer to these processes as clipped processes that are started from the
same distribution.  Let $P_{t}(x),P^j_t(x)$ and $Q_{t}(x)$ show the density of
$X_{t}, X^j_t$ and $X'_{t}$ at $x$.

We prove a bound on the probability of jump on both continuous time and
discrete time process that will be used to remove the bounded movement
assumption.  This part is an extension of Lemma~13 in \cite{ganesh2020faster}.
First we prove both continuous-time and discrete-time processes satisfy a
semi-contraction inequality.
\begin{lemma}
\label{lem:semi-contract}
If $f$ satisfies Assumption~\ref{as:all} and $z_t$ and $z'_t$ are two instances
of Langevin diffusion with synchronously coupled Brownian motion, then we have the following
\eqn{
    \norm{z_t - z'_t}
    \leq
    e^{-\growth t}\norm{z_0 - z'_0} + \sqrt{\frac{\pert}{\growth}\left(1-e^{-2 \growth t}\right)}.
}
Furthermore, if $\step \leq \growth/\lip^2,$ 
then gradient descent satisfies the following
\eqn{
    \norm{(x - \step \grad f(x)) - (y - \step \grad f(y))}^2 \leq (1-\growth \step)\norm{x-y}^2 + 2\step \pert.
}
\end{lemma}
\begin{proof}
For discrete-time, the result follow from elementary calculations. For
continuous-time, by coupling the Brownian motions and subtracting we get the
following,
\eqn{
    \frac{d}{d t} (z_t - z'_t) = - \left(\grad f(z_t) - \grad f(z'_t)\right),
}
then we differentiate $\norm{z_t - z'_t}^2$ with respect to time to get
\eqn{
    \frac{d}{d t} \norm{z_t - z'_t}^2
    =
    -2\inner{z_t - z'_t, \grad f(z_t) - \grad f(z'_t)}
    \leq
    -2 \growth \norm{z_t - z'_t}^2 + 2\pert,
}
where the last step follows from strong dissipativity.  Solving this differential
inequality and using $\sqrt{x+y}\leq \sqrt{x}+\sqrt{y},$ concludes the proof.
\end{proof}
\begin{lemma}
\label{lem:jump}
Suppose the potential satisfies Assumption~\ref{as:all} and
$x_0=z_0 \sim \Gsn(0,\sigma^2 \id),$
for $\sigma^2 < (\lip+1)^{-1}.$ If the step size is
small enough,
$\step \leq \frac{1 \wedge \growth}{4(1 \vee \lip^2)} \wedge
\frac{2}{\norm{\grad f(0)}^2},$
then each of the following jump conditions,
denoted with $\event^1_\delta$ and $\event^2_\delta$, happens with probability
at least $1 - \delta$.
\eqn{
    \forall t \leq N\step : \norm{\tx_{t} - \tx_{\floor{t/\step}\step}}
&\leq 
(2\cond + 1)
\left(1 + \sqrt{\pert} + \sqrt{d}+2\sqrt{\log{\left(2(N+1)/\delta\right)}}\right)\sqrt{2\step},
\\
\forall t \leq N\step : \norm{z_t - z_{\floor{t/\step}\step}}
&\leq
(5\cond + 1)
\left(1 + \sqrt{\pert} + \sqrt{d} +2\sqrt{\log{\left(2(N+1)/\delta\right)}} \right)\sqrt{2\step}
.
}
\end{lemma}
\begin{remark}
    The RHS of both of the bounds can be written as
    $c\cond\left(\sqrt{\pert}+\sqrt{d}+\sqrt{\log{N/\delta}}\right)\sqrt{\step}$,
    for a universal constant $c$ (note that $\cond >1$).
\end{remark}

\begin{proof}
In order to ease the notation we will use $B^t_u$ to denote $B_{t+u} - B_t$.
From Lemma~\ref{lem:sup-brown} we know that
\eqn{
\label{eq:sup-brown}
\sqrt{2} \sup_{s \leq \step} \norm{B^{k\step}_s}
\leq
\sqrt{2\step} \left( \sqrt{d} + 2 \sqrt{\log{\frac{2(N+1)}{\delta}}}\right),
}
with probability at least $1-\frac{\delta}{N+1},$ for all $k \leq N$.  We also
note that the initial distribution ${x_0 \sim \rho_0}$ satisfies the following (see
Lemma~\ref{lem:normal-tail})
\eqn{
    \P \left[
        \norm{x_0} \leq \frac{2\sqrt{2}}{\growth \sqrt{\step}}
        \left(
            1 + \sqrt{b} + \sqrt{d} + 2 \sqrt{\log{\frac{2(N+1)}{\delta}}}
        \right)
    \right]
    \geq
    1 - \frac{\delta}{N+1}.
}
First, we prove the discrete-time case.  By plugging $y=0$ into
Lemma~\ref{lem:semi-contract} and taking the square root, we get the following
\eqn{
\norm{x - \step \grad f(x)} 
\leq
(1- \frac{\growth \step}{2})\norm{x} + \sqrt{2\pert\step} + \step \norm{\grad f(0)},
}
thus, we can write
\eqn{
    \norm{x_{k+1}} 
    &\leq
    \norm{x_k-\step \grad f(x_k)} + \sqrt{2}\norm{B^{k\step}_{\step}}
    \\
    &\leq
    (1-\frac{\growth \step}{2})\norm{x_k}+\sqrt{2\step}(1+\sqrt{\pert}) + \sqrt{2}\sup_{s\leq\step} \norm{B^{k\step}_s},
}
where we used $\step^2\norm{\grad f(0)}^2 \leq 2\step$. This, combined with the
high probability bound on $x_0$ and supremum of Brownian motion with the aid of
union bound implies the following happens for $k\leq N$, with probability at
least $1-\delta$.
\eqn{
\norm{x_k} \leq \frac{2\sqrt{2}}{\growth \sqrt{\step}}
        \left(
            1 + \sqrt{\pert} + \sqrt{d} + 2 \sqrt{\log{\frac{2(N+1)}{\delta}}}
        \right).
}
Let $k=\floor{t/\step}$, we write
\eqn{
    \norm{\tx_{t}-\tx_{\step\floor{t/\step}}}
&\leq
\step \norm{\grad f(x_k)} + \sqrt{2} \sup_{t \leq \step} \norm{B^{k\step}_t}
\\
& \leq
\step \lip \norm{x_k} + \step \norm{\grad f(0)} + \sqrt{2} \sup_{t \leq \step} \norm{B^{k\step}_t}
\\
&\leq
\step \lip \norm{x_k} + \sqrt{2\step} + \sqrt{2} \sup_{t \leq \step} \norm{B^{k\step}_t}
,
}
combining this with the high probability bound on the Brownian motion and
$\norm{x_k}$ (note that there is no need for union bound as this event is
subset of the high probability event on the norm of the Brownian motion)
concludes the proof for the discrete case.

We use a similar structure for proving the continuous-time case.  Let $z'_u$
denote the continuous time Langevin dynamics started at $z'_0 = 0$.  We write
\eqn{
    \norm{z'_u}
    \leq
    \int_0^u \norm{\grad f(z'_s)} ds + \sqrt{2}\norm{B_u} 
    \leq
    \lip \int_0^u \norm{z'_s} ds + \step \norm{\grad f(0)} + \sup_{s \leq \step} \sqrt{2}\norm{B_s},
}
for $u \leq \step$.
Using \Gronwall inequality (Lemma~\ref{lem:gronwall}), we get the following
(for $u \leq \step$)
\eqn{
    \norm{z'_u}
    \leq
    e^{\lip u}
    \left( \step \norm{\grad f(0)} + \sup_{s \leq \step} \sqrt{2}\norm{B_s} \right)
    \leq
    2 \left( \step \norm{\grad f(0)} + \sup_{s \leq \step} \sqrt{2}\norm{B_s} \right).
}
where the last inequality is because $u \leq \step \leq \log{2}/\lip$.  Plug
$z'_0=0$ in Lemma~\ref{lem:semi-contract} and shift $z_0$ to $z_t$, using the
semi-group property, to get the following
\eqn{
    \label{eq:cont-step-mult}
    \norm{z_{t+u}}
    & \leq
    e^{-\growth u}\norm{z_t} + \norm{z'_u} + \sqrt{\frac{\pert}{\growth}\left(1-e^{-2 \growth u}\right)}
    \\
    &\leq
    (1-\growth u/2)\norm{z_t}
    +
    2
    \left(\sqrt{2\step} + \sup_{s \leq \eta} \sqrt{2}\norm{B^t_s} \right)
    +\sqrt{2 \pert \step},
}
where the last inequality follows from
$4 \growth u \leq 4 \growth \step \leq 1,$
and $\step^2 \norm{\grad f(0)}^2 \leq 2\step.$ We plug $u=\step$ in the
previous inequality and use the high probability bound on the Brownian motion
and $z_0 = x_0$ with union bound to get
$\norm{z_{k\step}} \leq
\frac{4\sqrt{2}}{\growth\sqrt{\step}}
(1+\sqrt{\pert}+\sqrt{d}+2\sqrt{\log{\frac{2(N+1)}{\delta}}})$
with probability at least $1-\delta$ for all $k \leq N$.

Next, we modify
\eqref{eq:cont-step-mult} as follows
\eqn{
    \norm{z_{t+u}}
    \leq
    \norm{z_t}
    +
    2
    \left(\sqrt{2\step} + \sup_{s \leq \eta} \sqrt{2}\norm{B_s} \right) + \sqrt{2 \pert \step}.
}
Using this inequality with previous high probability bound on
$\norm{z_{k\step}}$ shows that the following happens with probability at least
$1 - \delta$ for all $t \leq N \step$.
\eqn{
    \norm{z_t} \leq \frac{5\sqrt{2}}{\growth \sqrt{\step}}\left(1+\sqrt{\pert}+\sqrt{d}+2\sqrt{\log{\frac{2(N+1)}{\delta}}}\right),
}
where we used $\step \leq 1/4\growth$. Finally we use the following inequality
to connect this tail bound to probability of jump
\eqn{
    \norm{z_t - z_{\lfloor t/\step \rfloor \step}}
    & =
    \norm{\int_{\lfloor t/\step \rfloor \step}^t -\grad f(z_s)ds + \sqrt{2} dB_s }
    \\
    &\leq
    \step \lip \sup_{s \leq N\step} \norm{z_s} 
    + \step\norm{\grad f(0)}
    + \sqrt{2} \sup_{s\leq\step}\norm{B^{\lfloor t/\step \rfloor \step}_s}.
}
Combining previous inequality with high probability bounds on $\norm{z_t}$ and
the Brownian motion and using $\step^2\norm{\grad f(0)}^2 \leq 2\step$
concludes the proof for the continuous-time case (where again, there is no
need for union bound as this event is subset of the high probability event on
the norm of the Brownian motion).

Finally, we collect all the upper bounds on $\step$ in a more compact form
\eqn{
    \step
    \leq
    \frac{1 \wedge \growth}{4( 1 \vee \lip^2)}
    \wedge
    \frac{2}{\norm{\grad f(0)}^2}
    \leq
    \frac{2}{\norm{\grad f(0)}^2}
    \wedge
    \frac{4}{\growth\sigma^2}
    \wedge
    \frac{\growth}{\lip^2}
    \wedge
    \frac{1}{4\growth}
    \wedge
    \frac{\log{2}}{\lip}.
}
\end{proof}

Now we prove the continuous-time and discrete-time clipped processes stay
sufficiently close to each other.  In order to prove that we first state a
lemma that describes the behavior of discrete chain as step size approaches
zero.  We will show that the \Renyi divergence of the discrete-time process and
the continuous-time process converges to zero as the step size approaches zero.
The proof relies on the Girsanov theorem and
Lemma~\ref{lem:jump}.

\begin{lemma}
\label{lem:convergence}
Let $\Gamma^{\step}_{T}, \Pi_{T}$ be the distribution of the paths on $[0,T]$
of the interpolated time process with step-size $\step$, and continuous time
process respectively.  Starting from $x_0 = z_0 = \Gsn(0,\sigma^2 \id)$ for
$\sigma^2 \leq (\lip+1)^{-1},$ and under Assumption \ref{as:all}, for any
$T>0$, $\alpha \geq 1$
\eqn{
    \lim_{\step \to 0} \renyi{\alpha}{\Gamma_{T}^\step}{ \Pi_{T}} = 0.
}
\end{lemma}
\begin{proof}
Denote our underlying probability space by $(\Omega, \mathcal{F}, P)$;
let $\*Z = \{f \vert f:[0,T] \to \reals^d \}$ be the space of possible paths.
We use the notations $\tx_\cdot(\omega)$ and $z_\cdot(\omega)$ to denote
one realization of discrete and continuous time process.
Using Girsanov's Theorem (see Lemma~\ref{lem:girsanov-2}),
we define the measure $Q$ 
\eqn{
    \frac{dQ}{dP}(\omega)
    &=
    N(\tx_\cdot)\\
    &\defeq
    \exp \bigg(
        -\frac{1}{\sqrt{2}} \int_0^T \left(\grad f(\tx_s) - \grad f(\tx_{\floor{s/\step}\step})\right)^\top dB_s 
        -
        \frac{1}{4} \int_0^T \norm{\grad f(\tx_s) - \grad f(\tx_{\floor{s/\step}\step})}^2 ds
    \bigg),
}
such that the $P$-law of $z_\cdot$ equals the $Q$-law of $\tx_\cdot$.  Note
that Novikov condition can be checked by using Lemmas~\ref{lem:jump} and
\ref{lem:uncondition} and smoothness of potential.  Thus, we can write the
following
\eqn{
    \Pi_T(A)
    &=
    P\left(\{\omega \vert z_\cdot(\omega) \in A\}\right)
    \\
    &=
    Q\left(\{\omega \vert \tx_\cdot(\omega) \in A\}\right)
    \\
    &=
    \int \ind{\tx_\cdot(\omega) \in A} N(\tx_\cdot(\omega)) d P(\omega)
    \\
    &=
    \int_{\tx_\cdot \in A} N(\tx_\cdot) d \Gamma_T^\step(\tx_\cdot),
}
for any measurable $A \in \*Z$. This implies the following about the
Radon–Nikodym derivative
\eqn{
    \frac{d \Pi_T}{d \Gamma_T^\step}(\tx_.)
    =
    N(\tx_.).
}
Raising to power $-(\alpha -1)$ and using the definition of $N$ and taking
expectation we get
\eqn{
    \Esub{\left(\frac{d\Gamma_{T}^\eta}{d\Pi_{T}}\right)^\alpha}{\Pi_{T}}
    &=
    \Esub{\left(\frac{d\Gamma_{T}^\eta}{d\Pi_{T}}\right)^{\alpha-1}}{\Gamma^\step_{T}} 
    \\
    &=
    \mathbb{E} \bigg[\exp\bigg(
            \frac{\alpha-1}{\sqrt{2}} \int_0^T \left( \grad f(\tx_s) - \grad f(\tx_{\floor{s/\step}\step})\right)^\top dB_s +
            \\
            &\qquad\qquad \frac{\alpha-1}{4} \int_0^T \norm{\grad f(\tx_s) - \grad f(\tx_{\floor{s/\step}\step}) }^2 ds 
    \bigg)\bigg].
}
If we define $M_s = \sqrt{2}(\alpha-1) \left(\grad f(\tx_s)- \grad f(\tx_{\floor{s/\step}\step}) \right)$,
this expectation is equal to the following
\eqn{
    \EE{\exp\left(
        \frac{1}{2} \int_0^T M_s^\top dB_s - \frac{1}{4} \int_0^T \norm{M_s}^2 ds
        + \int_0^T \left(\frac{1}{4} + \frac{1}{8(\alpha-1)} \right) \norm{M_s}^2 ds
    \right)}.
}
Subsequently, the following holds by Cauchy-Schwartz 
\eqn{
    \Esub{\left(\frac{\Gamma_{T}^\eta}{\Pi_{T}}\right)^\alpha}{\Pi_{T}}
    &\leq
    \EE{\exp\left( \int_0^T M_s^\top dB_s - \frac{1}{2} \int_0^T \norm{M_s}^2 ds \right)}^{1/2}\\
        & \quad \quad \times
        \EE{\exp \left\{\left((\alpha-1)^2 + \frac{\alpha-1}{2} \right) \int_0^T \norm{ \grad f(\tx_s)- \grad f(\tx_{\floor{s/\step}\step})}^2 ds \right\}}^{1/2}
    \\
    &\leq
    \EE{\exp \left(\alpha^2 \int_0^T \norm{\grad f(\tx_s) - \grad f(\tx_{\floor{s/\step}\step})}^2 ds \right)}^{1/2},
}
where the last inequality step follows from Lemma~\ref{lem:novikov} and again
Novikov's condition holds as argued before.  We use Assumption \ref{as:all},
and the first event in Lemma \ref{lem:jump} to get
\eqn{
    \int_0^T \norm{\grad f(\tx_s) - \grad f(\tx_{\floor{s/\step}\step})}^2 ds
    &\leq
    \lip^2 \int_0^T \norm{\tx_s - \tx_{\floor{s/\step}\step}}^2 ds
    \\
    &\leq K T \lip^2 \cond ^2 \left(\pert + d + \log{\frac{T}{\step\delta}} \right) \step,
}
with probability at least $1-\delta$, if $\step,$ is sufficiently small for
some universal constant $K$.  Letting the event in which this bound holds be
called $\event_\delta$, by calculating the conditional expectation we get
\eqn{
    \mathbb{E}
    \Bigg[
    \exp \bigg(c \alpha^2 \int_0^T & \norm{\grad f(\tx_s) - \grad f(\tx_{\floor{s/\step}\step})}^2 ds \bigg)
    \vert \event_\delta
    \Bigg]
    \\
    & \leq
    \exp \left(K c \alpha^2 T \lip ^2 \cond^2
    \left(\pert + d +\log{\frac{T}{\step\delta}} \right) \step \right) \\
    &\leq \delta^{-\gamma_c(\step)} \exp \left(K c \alpha^2 T \lip^2 \cond^2
    \left(\pert + d + \log{T} - \log \step \right) \step \right), 
}
where $\gamma_c(\eta) = K c \alpha^2 T \lip^2 \cond^2 \step,$ and we absorbed
constants into $K$.  For any fixed $c,$ we know
$\lim_{\step \to 0} \gamma_c(\eta) = 0,$
therefore for small enough $\step,$ we have
$\gamma_c(\step) < 1,$ thus we can apply Lemma~\ref{lem:uncondition} with
$\theta = c$ and $\gamma = \gamma_c(\step)$ to get
\eqn{
    \EE{\exp \left(\alpha^2 \int_0^T \norm{\grad f(\tx_s) - \grad f(\tx_{\floor{s/\step}\step})}^2 ds \right)}
    \leq
     \frac{2^{\frac{2}{c}}c}{c-1} \exp \left(K \alpha^2 T \lip^2 \cond^2
    \left(\pert + d + \log T - \log \step \right) \step \right).
}    
We substitute this into our earlier bound,
\eqn{
    \Esub{\left(\frac{\Gamma_{T}^\eta}{\Pi_{T}}\right)^{\alpha}}{\Pi_T}
    &\leq
     \frac{2^{\frac{2}{c}}c}{c-1} \exp \left(K \alpha^2 T \lip^2 \cond^2
\left(\pert + d + \log T - \log \step \right) \step \right).
}
We take the limit as $\step \to 0$ to get
\eqn{
    \lim_{\step \to 0} \Esub{\left(\frac{\Gamma_{T}^\eta}{\Pi_{T}}\right)^{\alpha}}{\Pi_T}
     \leq
     \frac{2^{\frac{2}{c}}c}{c-1}.
}
Finally, we take another limit as $c \to \infty$ to get
\eqn{
    \lim_{\step \to 0} \Esub{\left(\frac{\Gamma_{T}^\eta}{\Pi_{T}}\right)^{\alpha}}{\Pi_T}
    \leq
    1.
}
Substituting this into the definition of the \Renyi divergence and using
continuity of $\log$ at $1$ concludes the proof.
\end{proof}
Now using the previous lemma, we extend \cite[Corollary~11]{ganesh2020faster}
to the interpolation process under strong dissipativity.
\begin{lemma}
\label{lem:clipped-ratio}
If $P_0 = Q_0 = \Gsn(0, \sigma^2 \id)$ for $\sigma^2 < (\lip + 1)^{-1},$ then
the following holds for $\alpha \geq 2,$ when Assumption~\ref{as:all} is
satisfied.
\eqn{\label{eq:clipped-ratio}
  \Esub{\frac{P_{T}}{Q_{T}}\left(x \right)^\alpha}{Q_{T}} 
    \leq
    \exp{\left( T \alpha (\alpha-1) \lip^2 r^2 \right)}
}
\end{lemma}
\begin{proof}
  We use the same argument as in \cite[Lemma~10]{ganesh2020faster}.  By
sampling both $X_t$ and $X'_t$ at multiples of $\step/j$ and at the final
moment($T$), we get the following tuples.
\eqn{
    &X_{0-T} = \{X_{i \step / j}\}_{ 0 \leq i  < j T/\step} + \{X_T\},
    &X^j_{0-T} = \{X^j_{i \step / j}\}_{ 0 \leq i  < j T/\step} + \{X^j_T\},
}
where by $+$ we mean appending the element to the end of the tuple.  In order
to use Lemma~ \ref{lem:renyi-additive}, we consider functions $\phi_1$ and
$\phi_2$ that append one new sample to the tuples of sampled clipped processes.
For example $\phi_1$ gets $\{X_{i \step / j}\}_{ 0 \leq i <k}$ and applies
Langevin update rule along with the clipping criteria for step size $\step/j$
using the gradient at the last multiple of $\step$ to produce
$\{X_{i \step / j}\}_{ 0 \leq i \leq k}$.
$\phi_2$ is defined similarly but uses gradient at
last multiple of $\step/j$. Note that we get $X_{0-T}$ and $X^j_{0-T}$ by
multiple applications of $\phi_1$ and $\phi_2$, except for the final iterate.
Assume $\tilde{X}$ is a deterministic tuple (i.e. point mass) we bound
$\renyi{\alpha}{\phi_1(\tilde{X})}{\phi_2(\tilde{X})}$.  If $\tilde{X}$
contains $\perp$ then this is zero, therefore we assume $\tilde{X}$ does not
contain jumps larger than $r$ and by data processing inequality
(Lemma~\ref{lem:data-processing}) we can ignore clipping done by $\phi_1$ and
$\phi_2$. Since $\tilde{X}$ was a point mass, both $\phi_1(\tilde{X})$ and
$\phi_2(\tilde{X})$ are Gaussians with possibly different means, which cannot
differ more than $\lip r\step/j$, because of smoothness of potential and
assumption that jumps are smaller than $r$. Thus, Lemma~\ref{lem:renyi-normal}
implies
$\renyi{\alpha}{\phi_1(\tilde{X})}{\phi_2(\tilde{X})} \leq \alpha
\lip^2 r^2 \step/4j$.
Let $T = k \step/j + t'$ such that $t' < \step/j$. We
can apply Lemma~\ref{lem:renyi-additive} for $k$ times to get that the \Renyi
divergence between $\{X_{i \step / j}\}_{ 0 \leq i  < j T/\step}$ and
$\{X^j_{i \step / j}\}_{ 0 \leq i  < j T/\step}$
is bounded by $\frac{\alpha \lip^2 r^2 (T-t')}{4}$.
Now modifying $\phi_1$ and $\phi_2$ to use time $t'$ instead of
$\step/j$, by the same argument as before and using
Lemma~\ref{lem:renyi-additive} once more, we can conclude that
\eqn{
    \renyi{\alpha}{X_{0-T}}{X^j_{0-T}}
    \leq
    \frac{\alpha \lip^2 r^2 T}{4}.
}
To go back to $P_T$ and $Q_T$ we write
\eqn{
    \renyi{\alpha}{P_T}{Q_{T}}
    \leq 
    \frac{\alpha - 0.5}{\alpha - 1} \renyi{2\alpha}{P_T}{P^j_{T}}
    + \renyi{2\alpha - 1}{P^j_{T}}{Q_{T}},
}
where we used Cauchy-Schwartz inequality.  Taking the limit as $j \to \infty$
and using $\alpha \geq 2,$ we get
\eqn{
    \renyi{\alpha}{P_T}{Q_{T}}
    \leq 
    2 \lim_{j \to \infty} \renyi{2\alpha}{P_T}{P^j_{T}}
    + \lim_{j \to \infty} \renyi{2\alpha-1}{P^j_{T}}{Q_{T}}.
}
The second term in RHS converges to $0$, since by data processing inequality
(see Lemma~\ref{lem:data-processing}) we have
\eqn{
    \lim_{j \to \infty} \renyi{2\alpha-1}{P^j_{T}}{Q_{T}}
    \leq
    \lim_{j \to \infty} \renyi{2\alpha-1}{\Gamma_{T}^{\eta/j}}{\Pi_{T}}
    =
    0,
}
where the last step is due to Lemma \ref{lem:convergence}.
Therefore, we get the following
\eqn{
    \renyi{\alpha}{P_{T}}{Q_{T}}
    \leq
    2 \lim_{j \to \infty} \renyi{2\alpha}{P_{T}}{P^j_{T}}
    \leq
    2 \lim_{j \to \infty} \renyi{2\alpha}{X_{0-T}}{X^j_{0-T}}
    \leq 
    T \alpha \lip^2 r^2,
}
where the second inequality follows from data processing inequality
(Lemma~\ref{lem:data-processing}).  This in turn implies
\eqref{eq:clipped-ratio}.
\end{proof}
Finally, we combine the previous results to go back to the
unclipped process.
\begin{lemma}
\label{lem:disc-cont-div}
Suppose Assumption~\ref{as:all} holds and
$x_0 = z_0 = \Gsn(0, \sigma^2 \id)$ for $\sigma^2 < (\lip + 1)^{-1}.$
If
$\step \leq \frac{1 \wedge \growth}{4(1 \vee \lip^2)} \wedge
\frac{2}{\norm{\grad f(0)}^2},$
and for some universal constant $c$ we have
$N \step^2 \leq \frac{1}{c \cond^2 \lip^2 \alpha^2 },$
then for any 
$T \leq N\step$ and $\alpha \geq 2$, we have the following
\eqn{
  \Esub{\frac{\trho_{T}}{\pi_T}\left(x\right)^{\frac{\alpha}{4}+\frac{1}{2}}}{\pi_T}
    \leq
    \frac{5\alpha + 10}{\sqrt{\alpha}}
    \times
    \exp{\left(c T \cond^2 \lip^2 \alpha^2 (\pert + d + \log{(N)})\step\right)}.
}
\end{lemma}
\begin{proof}
Considering the events $\event^1_{\delta_1}, \event^2_{\delta_2}$,
we plug the following value in \eqref{eq:clipped-ratio}
$$r = c\cond(\sqrt{\pert}+\sqrt{d}+\sqrt{\log{(N/\delta_1)}}+ \sqrt{\log{(N/\delta_2)}})\sqrt{\step},$$
where $c$ is a universal constant such that the remark after
Lemma~\ref{lem:jump} holds. This implies
$ r^2 \leq c \cond^2 \left(\pert+d + \log{(N/\delta_1)} + \log{(N/\delta_2)}\right)\step$,
where we absorbed universal constants into $c$.
We write
\eqn{
  \Esub{\frac{P_{T}}{\pi_T}\left(x\right)^\alpha \bigg| \event^2_{\delta_2}}{\pi_T}
&\leq
\frac{1}{1-\delta_2} \Esub{\frac{P_{T}}{Q_{T}}(x)^\alpha}{Q_{T}}
\\
&\leq
2 \exp{\left(T \alpha (\alpha -1) \lip^2 r^2 \right)}
\\
&\leq
\frac{2 \exp{\left(c T \cond^2 \lip^2 \alpha^2(\pert + d + \log{(N)})\step\right)}}{(\delta_1 \delta_2)^{c T \cond^2 \lip^2 \alpha^2 \step}},
}
where the first step follows from $\pi_T(x) \geq Q_{T}(x)$
(for $x \in \reals^d$), and the second from Lemma~\ref{lem:clipped-ratio}.
In order to utilize Lemma~\ref{lem:uncondition} we set
$\gamma = c T \cond^2 \lip^2 \alpha^2 \step$ and we
need $\gamma < 1$ which combined with $T \leq N \step$
shows that it is sufficient if we have
\eqn{\label{eq:cdn1}
    N \step^2 \leq \frac{1}{c \cond^2 \lip^2 \alpha^2}.
}
Lemma~\ref{lem:uncondition} implies
\eqn{\label{eq:first_half}
  \Esub{\frac{P_{T}}{\pi_T}\left(x\right)^{\frac{\alpha}{2}}}{\pi_T}
\leq
4\sqrt{2} \frac{\exp{\left(c T \cond^2 \lip^2 \alpha^2(\pert+d+\log{(N)})\step \right)}}{\delta_1^{c T \cond^2 \lip^2 \alpha^2 \step}},
}
where universal constants are again absorbed into $c$.  For replacing $P_{T}$
with $\trho_{T}$ we write
\eqn{
  \Esub{\frac{P_{T}}{\pi_T}\left(x\right)^{\frac{\alpha}{2}}}{\pi_T}
&=
\int_{\reals^d} \frac{P_{T}(x)^{\alpha/2}}{\pi_T(x)^{\alpha/2-1}}dx
\\
&=
\frac{\alpha}{2} \int_{\reals^d} \int_{0}^{P_{T}(x)} \frac{y^{\alpha/2-1}}{\pi_T(x)^{\alpha/2-1}}dy dx
\\
&=
\frac{\alpha}{2}\int_{\reals^d}
\left(
\int_{0}^{\trho_{T}(x)} \frac{y^{\alpha/2-1}\ind{y \leq P_{T}(x)}}{\pi_T(x)^{\alpha/2-1}} \times \frac{1}{\trho_{T}(x)}dy
\right)
\trho_{T}(x) dx
\\
&=
\frac{\alpha}{2}
\Esub{\frac{y^{\alpha/2-1}}{\pi_T(x)^{\alpha/2-1}} \bigg| y \leq P_{T}(x) }{x \sim \trho_{T}, y \sim \text{U}(0, \trho_{T}(x))}\\
&\ \ \ \times
\P_{x \sim \trho_{T}, y \sim \text{U}(0, \trho_{T}(x))} \left[ y \leq P_{T}(x)\right].
}
We consider RHS term by term. For the first term we write
\eqn{
\Esub{\frac{y^{\alpha/2-1}}{\pi_T(x)^{\alpha/2-1}} \bigg| y \leq P_{T}(x) }{x \sim \trho_{T}, y \sim \text{U}(0, \trho_{T}(x))}
=
\Esub{\frac{y^{\alpha/2-1}}{\pi_T(x)^{\alpha/2-1}} \bigg| \event_{\delta_1}}{x \sim \trho_{T}, y \sim \text{U}(0, \trho_{T}(x))}
}
with the right coupling between $y$ and the path $x_.$.
For the second term we have
\eqn{
    \P_{x \sim \trho_{T}, y \sim \text{U}(0, \trho_{T}(x))} \left[ y \leq P_{T}(x)\right]
    &=
    \int_{\reals^d} \int_{0}^{P_{T}(x)} \frac{dy}{\trho_{T}(x)} \trho_{T}(x) dx 
    = \int_{\reals^d} P_{T}(x) dx \geq 1 - \delta_1 \geq \frac{1}{2}.
}
Putting these together we get
\eqn{
    \Esub{\frac{y^{\alpha/2-1}}{\pi_T(x)^{\alpha/2-1}} \bigg| \event_{\delta_1}}{x \sim \trho_{T}, y \sim \text{U}(0, \trho_{T}(x))}
    &\leq
    \frac{4}{\alpha} \Esub{\frac{P_{T}}{\pi_T}(x)^{\alpha/2}}{\pi_T}
    \\
    & \leq
    \frac{16\sqrt{2}}{\alpha}
    \frac{\exp{\left(c T \cond^2 \lip^2 \alpha^2 (\pert+d+\log{(N)})\step \right)}}{\delta_1^{c T \cond^2 \lip^2 \alpha^2 \step}}.
}
Using Lemma~\ref{lem:uncondition} another time, we need to set
$\gamma = c T \cond^2 \lip^2 \alpha^2 \step.$
The condition $\gamma < 1$ is already satisfied
as we used this lemma before (with a different $c$). Therefore, we get the
following (universal constants are absorbed into $c$ again.)
\eqn{ 
    \Esub{\frac{y^{\frac{\alpha}{4}-\frac{1}{2}}}{\pi_T(x)^{\frac{\alpha}{4}-\frac{1}{2}}}}{x \sim \trho_{T}, y \sim \text{U}(0, \trho_{T}(x))}
    \leq
    \frac{2^{17/4}}{\sqrt{\alpha}} \exp{\left(c T \cond^2 \lip^2 \alpha^2 (\pert + d + \log{(N)})\step\right)}.
}
Finally, we write
\eqn{
    \Esub{\frac{\trho_{T}}{\pi_T}(x)^{\frac{\alpha}{4}+\frac{1}{2}}}{ \pi_T} &=
    (\frac{\alpha}{4}+\frac{1}{2})
    \Esub{\frac{y^{\frac{\alpha}{4}-\frac{1}{2}}}{\pi_T(x)^{\frac{\alpha}{4}-\frac{1}{2}}}}{x \sim \trho_{T}, y \sim \text{U}(0, \trho_{T}(x))}
    \\
    &\leq
    (\frac{\alpha}{4}+\frac{1}{2})
    \frac{2^{17/4}}{\sqrt{\alpha}} \exp{\left(c T \cond^2 \lip^2 \alpha^2  (\pert + d + \log{(N)})\step\right)}.
}
\end{proof}

%
\section{Logarithmic Sobolev Inequality under Assumption~\ref{as:all}}
\label{sec:lsi-all}
For some $\growth ,\pert>0$, assume that the following holds
\begin{align}
\inner{\grad f(x) - \grad f(y) , x-y} \geq \growth\|x-y\|^2 - \pert\ \ \text{ for all }\ \ x,y \in\reals^d.
\end{align}

Define the Lyapunov function $W(x) = \exp\{ \frac{\gamma}{2} \|x-x_*\|^2\}$ for
some $\gamma$ and a critical point $x_*$ of $f$. We have
\begin{align}
\grad W(x) = \gamma (x-x_*)W(x) \ \text{ and } \Delta W(x) = \gamma (d+ \gamma \|x-x_*\|^2)W(x)
\end{align}
and consequently
\begin{align}
  \frac{LW(x)}{W(x)} =& \gamma \left( d+\gamma \|x-x_*\|^2 - \inner{x-x_*, f(x)}\right)\\
  \leq & \gamma \left( d+\pert + (\gamma-\growth)\|x-x_*\|^2 \right).
\end{align}

Next, choosing $\gamma = \growth/2$ and defining
$R^2 = \frac{2}{\growth}(d+\pert+1),$
one can show that the right hand side above is upper bounded by
\begin{align}
  -\frac{\growth}{2} + \frac{\growth}{2}(d+\pert)\ind{\|x-x_*\| \leq R}.
\end{align}
Thus, the target $\target = e^{-f}$ satisfies the Lyapunov condition given in
\cite{bakry2008simple}, and consequently satisfies a \poincare inequality with a
constant upper bounded with
\begin{align}
\poin \leq \frac{2}{\growth}\left( 1 + c\frac{\growth}{2}(d+\pert) e^{\text{Osc}_R(f)}\right)
\end{align}
where
$\text{Osc}_R(f) = \sup_{\|x-x_*\| \leq R} f(x) - \inf_{\|x-x_*\|\leq R}f(x)$
and $c$ is a absolute constant.  This bound is of order
$\*{O}(d e^{\text{Osc}_R(f)})$.
Further using the results of \cite{cattiaux2010note},
one can show that LSI holds for this class of potentials with a constant bounded by $\mathcal{O}(d^2e^{\text{Osc}_R(f)})$.

\section{Useful Lemmas}
\begin{lemma}
\label{lem:renyi-normal}
(From \cite{van2014renyi})
The \Renyi divergence of two Gaussians can be calculated as follows
\eqn{
  \renyi{\alpha}{\Gsn(0, \sigma^2\id)}{\Gsn(x, \sigma^2\id)} = \frac{\alpha \norm{x}^2}{2 \sigma^2}.  
}
\end{lemma}
\begin{lemma}
\label{lem:data-processing}
(Data processing inequality, From \cite{van2014renyi})
Suppose  $x_1 \sim \rho_1$ and $x_2 \sim \rho_2$. For  any function
$f$, let $f(x_1) \sim \pi_1$ and $f(x_2) \sim \pi_2$, then
$   \renyi{\alpha}{\pi_1}{\pi_2} \leq 
    \renyi{\alpha}{\rho_2}{\rho_2}.$
\end{lemma}
\begin{lemma}
\label{lem:renyi-additive}
(From \cite{mironov2017renyi})
Let $\Delta(S_1), \Delta(S_2)$ be the space of probability measures on $S_1$, $S_2$
let $\phi_1, \phi_1':\Delta(S_1) \to \Delta(S_2) $ and
$\phi_2, \phi_2':\Delta(S_2) \to \mathfrak{P}$ be maps such that 
for any distributions $\delta$ that is a point mass (on either $\Delta(S_1)$ or $\Delta(S_2)$)
we have $\renyi{\alpha}{\phi_i(\delta)}{\phi_i'(\delta)} \leq \eps_i$. Then,
for any probability measure $\rho \in \Delta(S_1)$ we have
$
    \renyi{\alpha}{\phi_2(\phi_1(\rho))}{\phi_2'(\phi_1'(\rho))} \leq \eps_1 + \eps_2.
$
\end{lemma}
\begin{lemma}
(Adapted from \cite[Lemma~14]{ganesh2020faster})
\label{lem:uncondition}
Let $Y>0$ (a.s),$\gamma < 1$ and $\theta>1+\gamma$.
If for all $0<\delta<1/2$ an event $\event_\delta$ has probability at least $1-\delta$,
and $\EE{Y^\theta \vert \event_\delta} \leq \frac{\beta}{\delta^\gamma}$,
then $\EE{Y} \leq 2^{2/\theta} \beta^{1/\theta} \frac{\theta}{\theta-1}$.
In particular, if $\theta=2$, we get:
$\EE{Y} \leq 4 \sqrt{\beta}$.
\end{lemma}

\begin{lemma}
\label{lem:normal-tail}
For $W \sim \Gsn(0,\id)$ we have the following tail bound for $x \geq 0$
\eqn{
\P\left[ \norm{W} \geq \sqrt{d} + x  \right] \leq \exp{(-x^2/2)},
}
\end{lemma}
\begin{proof}
Suppose $W = (w_1,\ldots, w_d)$ and denote $\sqrt{d}+x$ with $a$. We write
\eqn{
\P\left[ \norm{W} \geq a \right]
=
\P\left[ \exp{\left( t \norm{W}^2 \right)} \geq \exp{(ta^2)}  \right]
\leq
\frac{\EE{\exp{(t w_1^2)}}^d}{\exp{(ta^2)}}
= \exp{\left(-t a^2 -\frac{d}{2}\ln{(1-2t)}\right)},
}
for all $t < 1/2$, therefore we can plug $t = \frac{1}{2}\left(1-\frac{d}{a^2}\right),$
and put $a=\sqrt{d}+x$ back to get
\eqn{
\P\left[ \norm{W} \geq \sqrt{d} + x  \right] 
\leq
\exp{(-x^2/2)}\times \exp{\left(- d \left(\frac{x}{\sqrt{d}} - \ln{(1+ \frac{x}{\sqrt{d}})} \right)\right)}
\leq
\exp{\left(-x^2/2\right)},
}
where the last inequality holds since
$\frac{x}{\sqrt{d}} - \ln{(1+ \frac{x}{\sqrt{d}})} \geq 0.$
\end{proof}
\begin{lemma}
\label{lem:sup-brown}
For $d$-dimensional Brownian motion $B_t$ we have ($x \geq 0$)
\eqn{
  \P\left[ \sup_{s \leq t} \norm{B_s} \geq \sqrt{t}(\sqrt{d}+x) \right] \leq 2\exp{(-x^2/4)}.
}
\end{lemma}
\begin{proof}
Let $r$ denote $\sqrt{t}(\sqrt{d}+x)$ and $\tau$ denote the first exit time of
$B_t$ out of the ball of radius $r$ around origin.
Note that $\tau<t$ coincides with $\sup_{s \leq t} \norm{B_s} > r$, furthermore $\norm{B_\tau}=r$.
We write
\eqn{
    \P\left({\sup_{s\leq t} \norm{B_s} \geq r} \right)
    &\leq
    \P\left(\norm{B_t} \geq r\right) + \P \left(\tau <t, \norm{B_t} < r\right)
    \\
    &=
    \P\left(\norm{B_t} \geq r\right) + \EE{\ind{\tau < t}\P\left(\norm{B_t - B_\tau + B_\tau} < r \vert \tau, B_\tau\right)}
    \\
    &\leq
    \P\left(\norm{B_t} \geq r\right) + \EE{\ind{\tau < t}\P\left(\inner{B_t - B_\tau, B_\tau} < 0 \vert \tau, B_\tau\right)}
    \\
    &=
    \P\left(\norm{B_t} \geq r\right) + \P\left({\sup_{s\leq t} \norm{B_s} \geq r} \right)/2,
}
where the last step follows from independence of updates and normality.
Rearranging and using Lemma~\ref{lem:normal-tail} concludes the proof.
\end{proof}

\begin{lemma}
\label{lem:gronwall}
(\Gronwall inequality \cite{bellman1943stability})
For a function $v$ satisfying
$v(t) \leq C + A \int_0^t v(s)ds,$ 
for $0 \leq t \leq T$ with $A>0$. The following holds:
$v(t) \leq C e^{At}$.
\end{lemma}

\begin{lemma}\label{lem:rec-bound}
For a real sequence $\{ \theta_k\}_{k\geq 0}$,
if we have ${\theta_{k} \leq (1-a) \theta_{k-1} + h}$ for some ${a \in(0,1)}$, and ${h \geq 0}$, then
$\theta_k \leq e^{-a k} \theta_0 + {h}/{a}.$
  \end{lemma}
\begin{proof}
  Recursion on ${\theta_{k} \leq (1-a) \theta_{k-1} + h}$ yields
  \begin{equation*}
    \theta_k \leq (1-a)^k \theta_0 + h(1+(1-a)+(1-a)^2+\dots+(1-a)^{k-1})
    \leq
    (1-a)^k \theta_0 + \frac{h}{a}.
  \end{equation*}
  Using the fact that ${1-a \leq e^{-a}}$ completes the proof.
\end{proof}

\begin{lemma}\label{lem:novikov}
(Exponential Martingale Theorem \cite[Chapter III, Theorem~5.3]{ikeda2014stochastic})
Let $B_t$ be a Brownian motion and $\*{F}_t$ its associated filtration.
If for an $\*{F}_t$-adapted stochastic process $M_t$ and some $T\geq0,$
the following (Novikov's) condition holds
$$
\EE{\exp{\left(\frac{1}{2}\int_{0}^T \norm{M_s}^2ds\right)}} < \infty,
$$
then $\exp \left( \int_0^t M_s^\top dB_s - \frac{1}{2} \int_0^t \norm{M_s}^2 ds \right)$
is an exponential Martingale and in particular its expectation is equal to $1$ for all $t \leq T$.
\end{lemma}

\begin{lemma}
(Girsanov Theorem, Adapted from \cite[Theorem 8.6.8]{oksendal2013stochastic})
\label{lem:girsanov-2}
Let $x_t, y_t \in \reals^d$ be defined as follows
\eqn{
    d x_t(\omega) &= b(x_t(\omega)) dt + \sqrt{2} d B_t(\omega),\\
    d y_t(\omega) &= \gamma(\omega, t) dt + \sqrt{2} d B_t(\omega),
}
such that $y_0 = x_0$ and $\omega$ is an element of underlying probability space $\Omega$. 
Let $\{\*F_t\}$ be the natural filtration for $B_t$ and $P$ be the measure such that
$B_t$ is Brownian with respect to $P$ and let
\eqn{
    M_t(\omega) \defeq \exp \left(
    -\frac{1}{\sqrt{2}} \int_0^t (\gamma(\omega, s) - b(y_s(\omega)))^\top d B_s(\omega)
    - \frac{1}{4}\int_0^t \norm{\gamma(\omega, s) - b(y_s(\omega))}^2 ds
\right).
}
If $M_t$ is a martingale with respect to $\*F_t$,
in particular if $\gamma(\omega,s) - b(y_s(\omega)),$ satisfies Novikov's condition, then
on $\*F_T$ we have a unique measure $Q$ such that
\eqn{
    \frac{dQ}{dP}(\omega)
    =
    M_T(\omega),
}
with the property that the $Q$-law of $y_\cdot$ is equal to the $P$-law of $x_\cdot$,
where $x_\cdot(\omega)$ and $y_\cdot(\omega)$ are one realization of $x_t$ and $y_t$ on $[0,T]$.
\end{lemma}

Finally, we state two helper lemmas. The first lemma shows the convergence of
continuous time process when the target satisfies LSI.
\begin{lemma}
  [Adapted from Theorem~3 in \cite{vempala2019rapid}]
  \label{lem:renyi-cont-decay}
  If $f=-\log{\target}$ satisfies \eqref{eq:LSI}, then the following
  holds
  \eqn{
    \log{\Esub{\frac{\pi_T}{\target}\left(x\right)^\alpha}{\target}}
    \leq
    e^{-\frac{2 T}{\alpha \poin}}
    \log{\Esub{\frac{\pi_0}{\target}\left(x\right)^\alpha}{\target}}.
  }
\end{lemma}
We briefly remark that the analog of the above lemma in Chi-squared divergence requires only the \eqref{eq:PI}.

In the second helper lemma, we prove that initializing with a normal
distribution with a sufficiently small variance will cause
$\Esub{\frac{\rho_0}{\target}\left(x\right)^\alpha}{\target}$ to be of order
$\tO(c^{\alpha d})$ for some constant $c$.
\begin{lemma}
  \label{lem:init}
  Suppose $f$ is $\lip$-smooth and $\alpha \geq 2$, then the following holds
  for $\target = e^{-f}$ and $\rho_0 = \Gsn(0, \sigma^2 \id)$ when $\sigma^2 < (\lip +1)^{-1}.$
  \eqn{
    \Esub{\frac{\rho_0}{\target}\left(x\right)^\alpha}{\target}
    \leq
    \frac{\exp{\left((\alpha-1)(f(0)+\frac{\norm{\grad f(0)}^2}{2}) \right)}}{\left(2 \pi \sigma^2 \right)^{\frac{\alpha d}{2}}}
    \left(
      \frac{2\pi}{\frac{\alpha}{\sigma^2} - (\alpha-1)(\lip+1)}
    \right)^{\frac{d}{2}}
    .
  }
\end{lemma}
\begin{remark}
  For the sake of simplicity, we use the following crude bound
  \eqn{
    \Esub{\frac{\rho_0}{\target}\left(x\right)^\alpha}{\target}
    \leq e^{\alpha d C_\sigma}
    \ \text{ with }\ C_\sigma =  1+\tfrac{f(0) +\|\grad f(0)\|^2}{d} - \log(\sigma^2[(1+L)\wedge 2\pi])
  }
  where $C_\sigma$ is a dimension free constant that does not depend on $\alpha$.
\end{remark}

\begin{proof}
  For any $x$ we have
  \eqn{
    f(x)
    \leq
    f(0) + \frac{\norm{\grad f(0)}^2}{2} +
    \left(\frac{L+1}{2}\right)\norm{x}^2.
  }
  Thus, we can write
  \eqn{
    \Esub{\frac{\rho_0}{\target}\left(x\right)^\alpha}{\target}
    &\leq
    \frac{1}{\left( 2 \pi \sigma^2 \right)^{\frac{\alpha d}{2}}}
    \int_{\reals^d} \exp{\left( -\frac{\alpha \norm{x}^2}{2\sigma^2} +(\alpha-1)f(x) \right)}dx
    \\
    &\leq
    \frac{\exp{\left((\alpha-1)(f(0)+\frac{\norm{\grad f(0)}^2}{2}) \right)}}{\left(2 \pi \sigma^2 \right)^{\frac{\alpha d}{2}}}
    \int_{\reals^d} \exp{\left(-\frac{1}{2}\left(\frac{\alpha}{\sigma^2} -(\alpha-1)(\lip+1) \right) \norm{x}^2 \right)}dx
    \\
    &\leq
    \frac{\exp{\left((\alpha-1)(f(0)+\frac{\norm{\grad f(0)}^2}{2}) \right)}}{\left(2 \pi \sigma^2 \right)^{\frac{\alpha d}{2}}}
    \left(
      \frac{2\pi}{\frac{\alpha}{\sigma^2} - (\alpha-1)(\lip+1)}
    \right)^{\frac{d}{2}}.
  }
\end{proof}

%
\pagebreak

\end{document}